\documentclass[english]{article}
\usepackage[T1]{fontenc}
\usepackage{color}
\usepackage{verbatim}
\usepackage{float}
\usepackage{url}
\usepackage{amsmath}
\usepackage{amsthm}
\usepackage{amssymb}
\usepackage[authoryear]{natbib}

\makeatletter
\theoremstyle{plain}
\newtheorem{thm}{\protect\theoremname}
  \theoremstyle{plain}
  \newtheorem{lem}[thm]{\protect\lemmaname}

 \usepackage[final]{nips_2016}

\usepackage[utf8]{inputenc} 
\usepackage[T1]{fontenc}    
\usepackage{hyperref}       
\usepackage{url}            
\usepackage{booktabs}       
\usepackage{amsfonts}       
\usepackage{nicefrac}       
\usepackage{microtype}      

\usepackage{graphicx} 
\usepackage{subfig} 

\usepackage{natbib}

\usepackage{algorithm}
\usepackage{algorithmic}

\usepackage{thm-restate}
\usepackage{thmtools}
\usepackage{wrapfig}
\usepackage{amsmath}
\usepackage{bm}

\newcommand{\lyxdot}{.}

\newcommand{\diag}{\mathop{\mathrm{diag}}}

\allowdisplaybreaks

\newcommand{\R}{\mathbb{R}} 
\newcommand{\M}{\mathcal{M}} 
\newcommand{\V}{\mathcal{V}} 
\renewcommand{\b}{\mathbf} 
\newcommand{\T}{^\top} 
\newcommand{\C}{\mathcal{C}} 
\newcommand{\F}{\mathcal{F}} 
\newcommand{\E}{\mathbb{E}} 
\renewcommand{\P}{\mathbb{P}} 
\renewcommand{\d}{\mathrm{d}} 
\newcommand{\diam}{\mathrm{diam}}
\newcommand{\Q}{\mathbb{Q}} 
\newcommand{\K}{\mathcal{K}} 
\newcommand{\X}{\mathcal{X}} 
\newcommand{\G}{\mathcal{G}} 
\newcommand{\Me}{\mathbb{M}} 

\newcommand*{\QEDB}{\hfill\ensuremath{\square}}%

\makeatother

\usepackage{babel}
  \providecommand{\lemmaname}{Lemma}
\providecommand{\theoremname}{Theorem}

\begin{document}

\title{Interpretable Distribution Features \\
with Maximum Testing Power}

\author{Wittawat Jitkrittum, \,\, Zolt{\'a}n Szab{\'o}, \,\, Kacper Chwialkowski, \,\, Arthur Gretton \\
\url{wittawatj@gmail.com} \\
\url{zoltan.szabo.m@gmail.com} \\
\url{kacper.chwialkowski@gmail.com} \\
\url{arthur.gretton@gmail.com} \\[2mm]
Gatsby Unit, University College London
\vspace{-3mm}}
\maketitle
\begin{abstract}
Two semimetrics on probability distributions are proposed, given as
the sum of differences of expectations of analytic functions evaluated
at spatial or frequency locations (i.e, features). The features are
chosen so as to maximize the distinguishability of the distributions,
by optimizing a lower bound on test power for a statistical test using
these features. The result is a parsimonious and interpretable indication
of how and where two distributions differ locally. We show that the
empirical estimate of the test power criterion converges with increasing
sample size, ensuring the quality of the returned features. In real-world
benchmarks on high-dimensional text and image data, linear-time tests
using the proposed semimetrics achieve comparable performance to the
state-of-the-art quadratic-time maximum mean discrepancy test, while
returning human-interpretable features that explain the test results.

\end{abstract}

\section{Introduction }

\vspace{-2mm}We address the problem of discovering features of distinct
probability distributions, with which they can most easily be distinguished.
The distributions may be in high dimensions, can differ in non-trivial
ways (i.e., not simply in their means), and are observed only through
i.i.d. samples. One application for such divergence measures is to
model criticism, where samples from a trained model are compared with
a validation sample: in the univariate case, through the KL divergence
\citep{CarParPol96}, or in the multivariate case, by use of the maximum
mean discrepancy (MMD) \citep{LloGha15}.  An alternative, interpretable
analysis of a multivariate difference in distributions may be obtained
by projecting onto a discriminative direction, such that the Wasserstein
distance on this projection is maximized \citep{MueJaa15}. Note that
both recent works require low dimensionality, either explicitly (in
the case of Lloyd and Gharamani, the function becomes difficult to
plot in more than two dimensions), or implicitly in the case of Mueller
and Jaakkola, in that a large difference in distributions must occur
in projection along a particular one-dimensional axis. Distances between
distributions in high dimensions may be more subtle, however, and
it is of interest to find interpretable, distinguishing features of
these distributions. %

In the present paper, we take a hypothesis testing approach to discovering
features which best distinguish two multivariate probability measures
$P$ and $Q$, as observed by samples $\mathsf{X}:=\{\mathbf{x}_{i}\}_{i=1}^{n}$
drawn independently and identically (i.i.d.) from $P$, and $\mathsf{Y}:=\{\mathbf{y}_{i}\}_{i=1}^{n}\subset\mathbb{R}^{d}$
from $Q.$ Non-parametric two-sample tests based on RKHS distances
\citep{Eric2008,FroLauLerRey12,Gretton2012} or energy distances \citep{Szekely2004,Baringhaus2004}
have as their test statistic an integral probability metric, the Maximum
Mean Discrepancy \citep{Gretton2012,SejSriGreFuk13}. For this metric,
a smooth witness function is computed, such that the amplitude is
largest where the probability mass differs most \citep[e.g. ][Figure 1]{Gretton2012}.
\citet{LloGha15} used this witness function to compare the model
output of the Automated Statistician \citep{Lloyd2014-ABCD} with
a reference sample, yielding a visual indication of where the model
fails. In high dimensions, however, the witness function cannot be
plotted, and is less helpful. Furthermore, the witness function does
not give an easily interpretable result for distributions with local
differences in their characteristic functions. A more subtle shortcoming
is that it does not provide a direct indication of the distribution
features which, when compared, would maximize test power - rather,
it is the witness function \emph{norm}, and (broadly speaking) its
\emph{variance} under the null, that determine test power. %

Our approach builds on the analytic representations of probability
distributions of \citet{Chwialkowski2015}, where differences in expectations
of analytic functions at particular spatial or frequency locations
are used to construct a two-sample test statistic, which can be computed
in linear time. Despite the differences in these analytic functions
being evaluated at random locations, the analytic tests have greater
power than linear time tests based on subsampled estimates of the
MMD \citep{Gretton2012a,Zaremba2013}. Our first theoretical contribution,
in Sec.\,\ref{sec:main_results}, is to derive a lower bound on the
test power, which can be maximized over the choice of test locations.
We propose two novel tests, both of which significantly outperform
the random feature choice of \citeauthor{Chwialkowski2015}. The (ME)
test evaluates the difference of mean embeddings at locations chosen
to maximize the test power lower bound (i.e., spatial features); unlike
the maxima of the MMD witness function, these features are directly
chosen to maximize the distinguishability of the distributions, and
take variance into account. The Smooth Characteristic Function (SCF)
test uses as its statistic the difference of the two smoothed empirical
characteristic functions, evaluated at points in the frequency domain
so as to maximize the same criterion (i.e., frequency features). Optimization
of the mean embedding kernels/frequency smoothing functions themselves
is achieved on a held-out data set with the same consistent objective. 

As our second theoretical contribution in Sec.\,\ref{sec:main_results},
we prove that the empirical estimate of the test power criterion asymptotically
converges to its population quantity uniformly over the class of Gaussian
kernels%
. Two important consequences follow: first, in testing, we obtain
a more powerful test with fewer features. Second, we obtain a parsimonious
and interpretable set of features that best distinguish the probability
distributions. In Sec.\,\ref{sec:experiments}, we provide experiments
demonstrating that the proposed linear-time tests greatly outperform
all previous linear time tests, and achieve performance that compares
to or exceeds the more expensive quadratic-time MMD test \citep{Gretton2012}.
Moreover, the new tests discover features of text data (NIPS proceedings)
and image data (distinct facial expressions) which have a clear human
interpretation, thus validating our feature elicitation procedure
in these challenging high-dimensional testing scenarios.\vspace{-2mm}

\section{ME and SCF tests\label{sec:me_scf_tests}}

\vspace{-2mm}

In this section, we review the ME and SCF tests \citep{Chwialkowski2015}
for two-sample testing. In Sec.\,\ref{sec:main_results}, we will
extend these approaches to learn features that optimize the power
of these tests. Given two samples $\mathsf{X}:=\{\mathbf{x}_{i}\}_{i=1}^{n},\mathsf{Y}:=\{\mathbf{y}_{i}\}_{i=1}^{n}\subset\mathbb{R}^{d}$
independently and identically distributed (i.i.d.) according to $P$
and $Q$, respectively, the goal of a two-sample test is to decide
whether $P$ is different from $Q$ on the basis of the samples. The
task is formulated as a statistical hypothesis test proposing a null
hypothesis $H_{0}:P=Q$ (samples are drawn from the same distribution)
against an alternative hypothesis $H_{1}:P\neq Q$ (the sample generating
distributions are different). A test calculates a test statistic $\hat{\lambda}_{n}$
from $\mathsf{X}$ and $\mathsf{Y}$, and rejects $H_{0}$ if $\hat{\lambda}_{n}$
exceeds a predetermined test threshold (critical value). The threshold
is given by the $(1-\alpha)$-quantile of the (asymptotic) distribution
of $\hat{\lambda}_{n}$ under $H_{0}$ i.e., the null distribution,
and $\alpha$ is the significance level of the test. 

\textbf{ME test} \,The ME test uses as its test statistic $\hat{\lambda}_{n}$,
a form of Hotelling's T-squared statistic, defined as $\hat{\lambda}_{n}:=n\mathbf{\overline{z}}_{n}^{\top}\mathbf{S}_{n}^{-1}\mathbf{\overline{z}}_{n},$
where $\mathbf{\overline{z}}_{n}:=\frac{1}{n}\sum_{i=1}^{n}\mathbf{z}_{i}$,
$\mathbf{S}_{n}:=\frac{1}{n-1}\sum_{i=1}^{n}(\mathbf{z}_{i}-\mathbf{\overline{z}}_{n})(\mathbf{z}_{i}-\mathbf{\overline{z}}_{n})^{\top}$,
and $\mathbf{z}_{i}:=(k(\mathbf{x}_{i},\mathbf{v}_{j})-k(\mathbf{y}_{i},\mathbf{v}_{j}))_{j=1}^{J}\in\mathbb{R}^{J}.$
The statistic depends on a positive definite kernel $k:\mathcal{X}\times\mathcal{X}\to\mathbb{R}$
(with $\mathcal{X}\subseteq\mathbb{R}^{d}$), and a set of $J$ test
locations $\mathcal{V}=\{\mathbf{v}_{j}\}_{j=1}^{J}\subset\mathbb{R}^{d}$.
Under $H_{0}$, asymptotically $\hat{\lambda}_{n}$ follows $\chi^{2}(J)$,
a chi-squared distribution with $J$ degrees of freedom. The ME test
rejects $H_{0}$ if $\hat{\lambda}_{n}>T_{\alpha}$, where the test
threshold $T_{\alpha}$ is given by the $(1-\alpha)$-quantile of
the asymptotic null distribution $\chi^{2}(J)$. Although the distribution
of $\hat{\lambda}_{n}$ under $H_{1}$ was not derived, \citet{Chwialkowski2015}
showed that if $k$ is analytic, integrable and characteristic (in
the sense of \citet{Sriperumbudur2011}), under $H_{1}$, $\hat{\lambda}_{n}$
can be arbitrarily large as $n\to\infty$, allowing the test to correctly
reject $H_{0}$. 

One can intuitively think of the ME test statistic as a squared normalized
(by the inverse covariance $\mathbf{S}_{n}^{-1}$) $L^{2}(\mathcal{X},V_{J})$
distance of the mean embeddings \citep{Smola2007} of the empirical
measures $P_{n}:=\frac{1}{n}\sum_{i=1}^{n}\delta_{\mathbf{x}_{i}}$,
and $Q_{n}:=\frac{1}{n}\sum_{i=1}^{n}\delta_{\mathbf{y}_{i}}$ where
$V_{J}:=\frac{1}{J}\sum_{i=1}^{J}\delta_{\mathbf{v}_{i}}$, and $\delta_{\mathbf{x}}$
is the Dirac measure concentrated at $\mathbf{x}$. The unnormalized
counterpart (i.e., without $\mathbf{S}_{n}^{-1}$) was shown by \citet{Chwialkowski2015}
to be a metric on the space of probability measures for any $\mathcal{V}$.
Both variants behave similarly for two-sample testing, with the normalized
version being a semimetric having a more computationally tractable
null distribution, i.e., $\chi^{2}(J)$. 

\textbf{SCF test} \,The SCF uses the test statistic which has the
same form as the ME test statistic with a modified $\mathbf{z}_{i}:=[\hat{l}(\mathbf{x}_{i})\sin(\mathbf{x}_{i}^{\top}\mathbf{v}_{j})-\hat{l}(\mathbf{y}_{i})\sin(\mathbf{y}_{i}^{\top}\mathbf{v}_{j}),\hat{l}(\mathbf{x}_{i})\cos(\mathbf{x}_{i}^{\top}\mathbf{v}_{j})-\hat{l}(\mathbf{y}_{i})\cos(\mathbf{y}_{i}^{\top}\mathbf{v}_{j})]_{j=1}^{J}\in\mathbb{R}^{2J},$
where $\hat{l}(\mathbf{x})=\int_{\mathbb{R}^{d}}\exp(-i\mathbf{u}^{\top}\mathbf{x})l(\mathbf{u})\thinspace\mathrm{d}\mathbf{u}$
is the Fourier transform of $l(\mathbf{x})$, and $l:\mathbb{R}^{d}\to\mathbb{R}$
is an analytic translation-invariant kernel i.e., $l(\mathbf{x}-\mathbf{y})$
defines a positive definite kernel for $\mathbf{x}$ and \textbf{$\mathbf{y}$}.
In contrast to the ME test defining the statistic in terms of spatial
locations, the locations $\mathcal{V}=\{\mathbf{v}_{j}\}_{j=1}^{J}\subset\mathbb{R}^{d}$
in the SCF test are in the frequency domain. As a brief description,
let $\varphi_{P}(\mathbf{w}):=\mathbb{E}_{\mathbf{x}\sim P}\exp(i\mathbf{w}^{\top}\mathbf{x})$
be the characteristic function of $P$. Define a smooth characteristic
function as $\phi_{P}(\mathbf{v})=\int_{\mathbb{R}^{d}}\varphi_{P}(\mathbf{w})l(\mathbf{v}-\mathbf{w})\thinspace\mathrm{d}\mathbf{w}$
\citep[Definition 2]{Chwialkowski2015}. Then, similar to the ME test,
the statistic defined by the SCF test can be seen as a normalized
(by $\mathbf{S}_{n}^{-1}$) version of $L^{2}(\mathcal{X},V_{J})$
distance of empirical $\phi_{P}(\mathbf{v})$ and $\phi_{Q}(\mathbf{v})$.
The SCF test statistic has asymptotic distribution $\chi^{2}(2J)$
under $H_{0}$. We will use $J'$ to refer to the degrees of freedom
of the chi-squared distribution i.e., $J'=J$ for the ME test, and
$J'=2J$ for the SCF test.

In this work, we modify the statistic with a regularization parameter
$\gamma_{n}>0$, giving $\hat{\lambda}_{n}:=n\mathbf{\overline{z}}_{n}^{\top}\left(\mathbf{S}_{n}+\gamma_{n}I\right)^{-1}\mathbf{\overline{z}}_{n}$,
for stability of the matrix inverse. Using multivariate Slutsky's
theorem, under $H_{0}$, $\hat{\lambda}_{n}$ still asymptotically
follows $\chi^{2}(J')$ provided that $\gamma_{n}\to0$ as $n\to\infty$.

\section{Lower bound on test power, consistency of empirical power statistic\label{sec:main_results}}

\vspace{-2mm}This section contains our main results. We propose to
optimize the test locations $\mathcal{V}$ and kernel parameters (jointly
referred to as $\theta$) by maximizing a lower bound on the test
power in Proposition\,\ref{prop:lb_me_power}. This criterion offers
a simple objective function for fast parameter tuning. %
The bound may be of independent interest in other Hotelling's T-squared
statistics, since apart from the Gaussian case \citep[e.g.][Ch. 8]{Bilodeau2008},
the characterization of such statistics under the alternative distribution
is challenging. The optimization procedure is given in Sec.\,\ref{sec:experiments}.
We use $\mathbb{E}_{\mathbf{xy}}$ as a shorthand for $\mathbb{E}_{\mathbf{x}\sim P}\mathbb{E}_{\mathbf{y}\sim Q}$
and let $\|\cdot\|_{F}$ be the Frobenius norm. 

\begin{restatable}[Lower bound on ME test power]{prop}{lbmepower}
\label{prop:lb_me_power}

Let $\mathcal{K}$ be a uniformly bounded (i.e., $\exists B<\infty$
such that $\sup_{k\in\mathcal{K}}\sup_{(\mathbf{x},\mathbf{y})\in\mathcal{X}^{2}}|k(\mathbf{x},\mathbf{y})|\le B$)
family of $k:\mathcal{X}\times\mathcal{X}\to\mathbb{R}$ measurable
kernels. Let $\mathbb{V}$ be a collection in which each element is
a set of $J$ test locations. Assume that $\tilde{c}:=\sup_{\mathcal{V}\in\mathbb{V},k\in\mathcal{K}}\|\boldsymbol{\Sigma}^{-1}\|_{F}<\infty$.
For large $n$, the test power $\mathbb{P}\left(\hat{\lambda}_{n}\ge T_{\alpha}\right)$
of the ME test satisfies $\mathbb{P}\left(\hat{\lambda}_{n}\ge T_{\alpha}\right)\ge L(\lambda_{n})$
where 
\begin{align*}
L(\lambda_{n}) & :=1-2e^{-\xi_{1}(\lambda_{n}-T_{\alpha})^{2}/n}-2e^{-\frac{\left[\gamma_{n}(\lambda_{n}-T_{\alpha})(n-1)-\xi_{2}n\right]^{2}}{\xi_{3}n(2n-1)^{2}}}-2e^{-\left[(\lambda_{n}-T_{\alpha})/3-\overline{c}_{3}n\gamma_{n}\right]^{2}\gamma_{n}^{2}/\xi_{4}},
\end{align*}
and $\overline{c}_{3},\xi_{1},\ldots\xi_{4}$ are positive constants
depending on only $B,J$ and $\tilde{c}$. The parameter $\lambda_{n}:=n\boldsymbol{\mu}^{\top}\boldsymbol{\Sigma}^{-1}\boldsymbol{\mu}$
is the population counterpart of $\hat{\lambda}_{n}:=n\mathbf{\overline{z}}_{n}^{\top}\left(\mathbf{S}_{n}+\gamma_{n}I\right)^{-1}\mathbf{\overline{z}}_{n}$
where $\boldsymbol{\mu}=\mathbb{E}_{\mathbf{x}\mathbf{y}}[\mathbf{z}_{1}]$
and $\boldsymbol{\Sigma}=\mathbb{E}_{\mathbf{x}\mathbf{y}}[(\mathbf{z}_{1}-\boldsymbol{\mu})(\mathbf{z}_{1}-\boldsymbol{\mu})^{\top}]$.
For large $n$, $L(\lambda_{n})$ is increasing in $\lambda_{n}$. 

\end{restatable}

\begin{proof}[Proof (sketch)]
The idea is to construct a bound for $|\hat{\lambda}_{n}-\lambda_{n}|$
which involves bounding $\|\overline{\mathbf{z}}_{n}-\boldsymbol{\mu}\|_{2}$
and $\|\mathbf{S}_{n}-\boldsymbol{\Sigma}\|_{F}$ separately using
Hoeffding's inequality. The result follows after a reparameterization
of the bound on $\mathbb{P}(|\hat{\lambda}_{n}-\lambda_{n}|\ge t)$
to have $\mathbb{P}\left(\hat{\lambda}_{n}\ge T_{\alpha}\right)$.
See Sec.\,\ref{sec:proof_lb_me_power} for details.
\end{proof}
Proposition\,\ref{prop:lb_me_power} suggests that for large $n$
it is sufficient to maximize $\lambda_{n}$ to maximize a lower bound
on the ME test power. The same conclusion holds for the SCF test (result
omitted due to space constraints). Assume that $k$ is characteristic
\citep{Sriperumbudur2011}. It can be shown that $\lambda_{n}=0$
if and only if $P=Q$ i.e., $\lambda_{n}$ is a semimetric for $P$
and $Q$. In this sense, one can see $\lambda_{n}$ as encoding the
ease of rejecting $H_{0}$. The higher $\lambda_{n}$, the easier
for the test to correctly reject $H_{0}$ when $H_{1}$ holds. This
observation justifies the use of $\lambda_{n}$ as a maximization
objective for parameter tuning.

\textbf{Contributions} \,The statistic $\hat{\lambda}_{n}$ for both
ME and SCF tests depends on a set of test locations $\mathcal{V}$
and a kernel parameter $\sigma$. We propose to set $\theta:=\{\mathcal{V},\sigma\}=\arg\max_{\theta}\lambda_{n}=\arg\max_{\theta}\boldsymbol{\mu}^{\top}\boldsymbol{\Sigma}^{-1}\boldsymbol{\mu}$.
The optimization of $\theta$ brings two benefits: first, it significantly
increases the probability of rejecting $H_{0}$ when $H_{1}$ holds;
second, the learned test locations act as discriminative features
allowing an interpretation of how the two distributions differ. We
note that optimizing parameters by maximizing a test power proxy \citep{Gretton2012a}
is valid under both $H_{0}$ and $H_{1}$ as long as the data used
for parameter tuning and for testing are disjoint. If $H_{0}$ holds,
then $\theta=\arg\max0$ is arbitrary. Since the test statistic asymptotically
follows $\chi^{2}(J')$ for any $\theta$, the optimization does not
change the null distribution. Also, the rejection threshold $T_{\alpha}$
depends on only $J'$ and is independent of $\theta$.

To avoid creating a dependency between $\theta$ and the data used
for testing (which would affect the null distribution), we split the
data into two disjoint sets. Let $\mathsf{D}:=(\mathsf{X},\mathsf{Y})$
and $\mathsf{D}^{tr},\mathsf{D}^{te}\subset\mathsf{D}$ such that
$\mathsf{D}^{tr}\cap\mathsf{D}^{te}=\emptyset$ and $\mathsf{D}^{tr}\cup\mathsf{D}^{te}=\mathsf{D}$.
In practice, since $\boldsymbol{\mu}$ and $\boldsymbol{\Sigma}$
are unknown, we use $\hat{\lambda}_{n/2}^{tr}$ in place of $\lambda_{n}$,
where $\hat{\lambda}_{n/2}^{tr}$ is the test statistic computed on
the training set $\mathsf{D}^{tr}$. For simplicity, we assume that
each of $\mathsf{D}^{tr}$ and $\mathsf{D}^{te}$ has half of the
samples in $\mathsf{D}$. We perform an optimization of $\theta$
with gradient ascent algorithm on $\hat{\lambda}_{n/2}^{tr}(\theta)$.
The actual two-sample test is performed using the test statistic $\hat{\lambda}_{n/2}^{te}(\theta)$
computed on $\mathsf{D}^{te}$. The full procedure from tuning the
parameters to the actual two-sample test is summarized in Sec.\,\ref{sec:Algorithm}. 

Since we use an empirical estimate $\hat{\lambda}_{n/2}^{tr}$ in
place of $\lambda_{n}$ for parameter optimization, we give a finite-sample
bound in Theorem\,\ref{thm:lambda_conv_me} guaranteeing the convergence
of $\overline{\mathbf{z}}_{n}^{\top}(\mathbf{S}_{n}+\gamma_{n}I)^{-1}\overline{\mathbf{z}}_{n}$
to $\boldsymbol{\mu}^{\top}\boldsymbol{\Sigma}^{-1}\boldsymbol{\mu}$
as $n$ increases, uniformly over all kernels $k\in\mathcal{K}$ (a
family of uniformly bounded kernels) and all test locations in an
appropriate class. Kernel classes satisfying conditions of Theorem\,\ref{thm:lambda_conv_me}
include the widely used isotropic Gaussian kernel class $\mathcal{K}_{g}=\left\{ k_{\sigma}:(\mathbf{x},\mathbf{y})\mapsto\exp\left(-(2\sigma^{2})^{-1}\|\mathbf{x}-\mathbf{y}\|^{2}\right)\mid\sigma>0\right\} $,
and the more general full Gaussian kernel class $\mathcal{K}_{\mathrm{full}}=\{k:(\mathbf{x},\mathbf{y})\mapsto\exp\left(-(\mathbf{x}-\mathbf{y})^{\top}\mathbf{A}(\mathbf{x}-\mathbf{y})\right)\mid\mathbf{A}\text{ is positive definite}\}$
(see Lemma\,\ref{lemma:isotropic-Gaussian-properties} and Lemma\,\ref{lemma:full-Gaussian-properties}).

\begin{restatable}[Consistency of $\hat{\lambda}_n$  in the ME test]{thm}{consistencyme}
\label{thm:lambda_conv_me}

Let $\mathcal{X}\subseteq\mathbb{R}^{d}$ be a measurable set, and
$\mathbb{V}$ be a collection in which each element is a set of $J$
test locations. All suprema over $\mathcal{V}$ and $k$ are to be
understood as $\sup_{\mathcal{V}\in\mathbb{V}}$ and $\sup_{k\in\mathcal{K}}$
respectively. For a class of kernels $\mathcal{K}$ on $\mathcal{X}\subseteq\mathbb{R}^{d}$,
define 
\begin{align}
\mathcal{F}_{1} & :=\{\mathbf{x}\mapsto k(\mathbf{x},\mathbf{v})\mid k\in\mathcal{K},\mathbf{v}\in\mathcal{X}\},\quad\mathcal{F}_{2}:=\{\mathbf{x}\mapsto k(\mathbf{x},\mathbf{v})k(\mathbf{x},\mathbf{v}')\mid k\in\mathcal{K},\mathbf{v},\mathbf{v}'\in\mathcal{X}\},\label{eq:F_1-2}\\
\mathcal{F}_{3} & :=\{(\mathbf{x},\mathbf{y})\mapsto k(\mathbf{x},\mathbf{v})k(\mathbf{y},\mathbf{v}')\mid k\in\mathcal{K},\mathbf{v},\mathbf{v}'\in\mathcal{X}\}.\label{eq:F_3}
\end{align}

Assume that (1) $\mathcal{K}$ is a uniformly bounded (by $B$) family
of $k:\mathcal{X}\times\mathcal{X}\to\mathbb{R}$ measurable kernels,
(2) $\tilde{c}:=\sup_{\mathcal{V},k}\|\boldsymbol{\Sigma}^{-1}\|_{F}<\infty$,
and (3) $\mathcal{F}_{i}=\{f_{\theta_{i}}\mid\theta_{i}\in\Theta_{i}\}$
is VC-subgraph with VC-index $VC(\mathcal{F}_{i})$, and $\theta\mapsto f_{\theta_{i}}(m)$
is continuous $(\forall m,i=1,2,3)$. Let $\overline{c}_{1}:=4B^{2}J\sqrt{J}\tilde{c},\overline{c}_{2}:=4B\sqrt{J}\tilde{c},$
and $\overline{c}_{3}:=4B^{2}J\tilde{c}^{2}$. Let $C_{i}$-s $(i=1,2,3)$
be the universal constants associated to $\mathcal{F}_{i}$-s according
to Theorem 2.6.7 in \citet{Vaart2000}. Then for any $\delta\in(0,1)$
with probability at least $1-\delta$,{\small{}
\begin{align*}
 & \sup_{\mathcal{V},k}\left|\overline{\mathbf{z}}_{n}^{\top}(\mathbf{S}_{n}+\gamma_{n}I)^{-1}\overline{\mathbf{z}}_{n}-\boldsymbol{\mu}^{\top}\boldsymbol{\Sigma}^{-1}\boldsymbol{\mu}\right|\\
 & \le2T_{\mathcal{F}_{1}}\left(\frac{2}{\gamma_{n}}\overline{c}_{1}BJ\frac{2n-1}{n-1}+\overline{c}_{2}\sqrt{J}\right)+\frac{2}{\gamma_{n}}\overline{c}_{1}J(T_{\mathcal{F}_{2}}+T_{\mathcal{F}_{3}})+\frac{8}{\gamma_{n}}\frac{\overline{c}_{1}B^{2}J}{n-1}+\overline{c}_{3}\gamma_{n},\text{where}\\
T_{\mathcal{F}_{j}} & =\frac{16\sqrt{2}B^{\zeta_{j}}}{\sqrt{n}}\left(2\sqrt{\log\left[C_{j}\times VC(\mathcal{F}_{j})(16e)^{VC(\mathcal{F}_{j})}\right]}+\frac{\sqrt{2\pi[VC(\mathcal{F}_{j})-1]}}{2}\right)+B^{\zeta_{j}}\sqrt{\frac{2\log(5/\delta)}{n}},
\end{align*}
}for $j=1,2,3$ and $\zeta_{1}=1,\zeta_{2}=\zeta_{3}=2$. \end{restatable}
\begin{proof}[Proof (sketch)]
 The idea is to lower bound the difference with an expression involving
$\sup_{\mathcal{V},k}\|\overline{\mathbf{z}}_{n}-\boldsymbol{\mu}\|_{2}$
and $\sup_{\mathcal{V},k}\|\mathbf{S}_{n}-\boldsymbol{\Sigma}\|_{F}$.
These two quantities can be seen as suprema of empirical processes,
and can be bounded by Rademacher complexities of their respective
function classes (i.e., $\mathcal{F}_{1},\mathcal{F}_{2},$ and $\mathcal{F}_{3}$).
Finally, the Rademacher complexities can be upper bounded using Dudley
entropy bound and VC subgraph properties of the function classes.
Proof details are given in Sec.\,\ref{sec:proof_lambda_conv_me}.
\end{proof}
Theorem\,\ref{thm:lambda_conv_me} implies that if we set $\gamma_{n}=\mathcal{O}(n^{-1/4})$,
then we have $\sup_{\mathcal{V},k}\left|\overline{\mathbf{z}}_{n}^{\top}(\mathbf{S}_{n}+\gamma_{n}I)^{-1}\overline{\mathbf{z}}_{n}-\boldsymbol{\mu}^{\top}\boldsymbol{\Sigma}^{-1}\boldsymbol{\mu}\right|=\mathcal{O}_{p}(n^{-1/4})$
as the rate of convergence. Both Proposition \ref{prop:lb_me_power}
and Theorem \ref{thm:lambda_conv_me} require $\tilde{c}:=\sup_{\mathcal{V}\in\mathbb{V},k\in\mathcal{K}}\|\boldsymbol{\Sigma}^{-1}\|_{F}<\infty$
as a precondition. To guarantee that $\tilde{c}<\infty$, a concrete
construction of $\mathcal{K}$ is the isotropic Gaussian kernel class
$\mathcal{K}_{g}$, where $\sigma$ is constrained to be in a compact
set. Also, consider $\mathbb{V}:=\{\mathcal{V}\mid\text{any two }\text{ locations are at least }\epsilon\text{ distance apart, and all test locations have their norms bounded by }\zeta\}$
for some $\epsilon,\zeta>0$. Then, for any non-degenerate $P,Q$,
we have $\tilde{c}<\infty$ since $(\sigma,\mathcal{V})\mapsto\lambda_{n}$
is continuous, and thus attains its supremum over compact sets $\mathcal{K}$
and $\mathbb{V}$.

\section{Experiments\label{sec:experiments} }

\begin{wraptable}{r}{0.50\textwidth}   
\vspace{-2mm}
\caption{Four toy problems. $H_{0}$ holds only in SG.\label{tab:toy_problems}}
\label{tab:toy_problems}
\centering{}%
\begin{tabular}{lll} 
\toprule
Data & $P$ & $Q$\tabularnewline 
\midrule
SG & $\mathcal{N}(\boldsymbol{0}_{d},I_{d})$ & $\mathcal{N}(\boldsymbol{0}_{d},I_{d})$\tabularnewline 
GMD & $\mathcal{N}(\boldsymbol{0}_{d},I_{d})$ & $\mathcal{N}((1,0,\ldots,0)^{\top},I_{d})$\tabularnewline 
GVD & $\mathcal{N}(\boldsymbol{0}_{d},I_{d})$ & $\mathcal{N}(\boldsymbol{0}_{d},\diag(2,1,\ldots,1))$\tabularnewline 
Blobs & \multicolumn{2}{p{55mm}}{Gaussian mixtures in $\mathbb{R}^{2}$ as studied in \citet{Chwialkowski2015,Gretton2012a}. }\tabularnewline 
\multicolumn{3}{c}{\includegraphics[width=30mm]{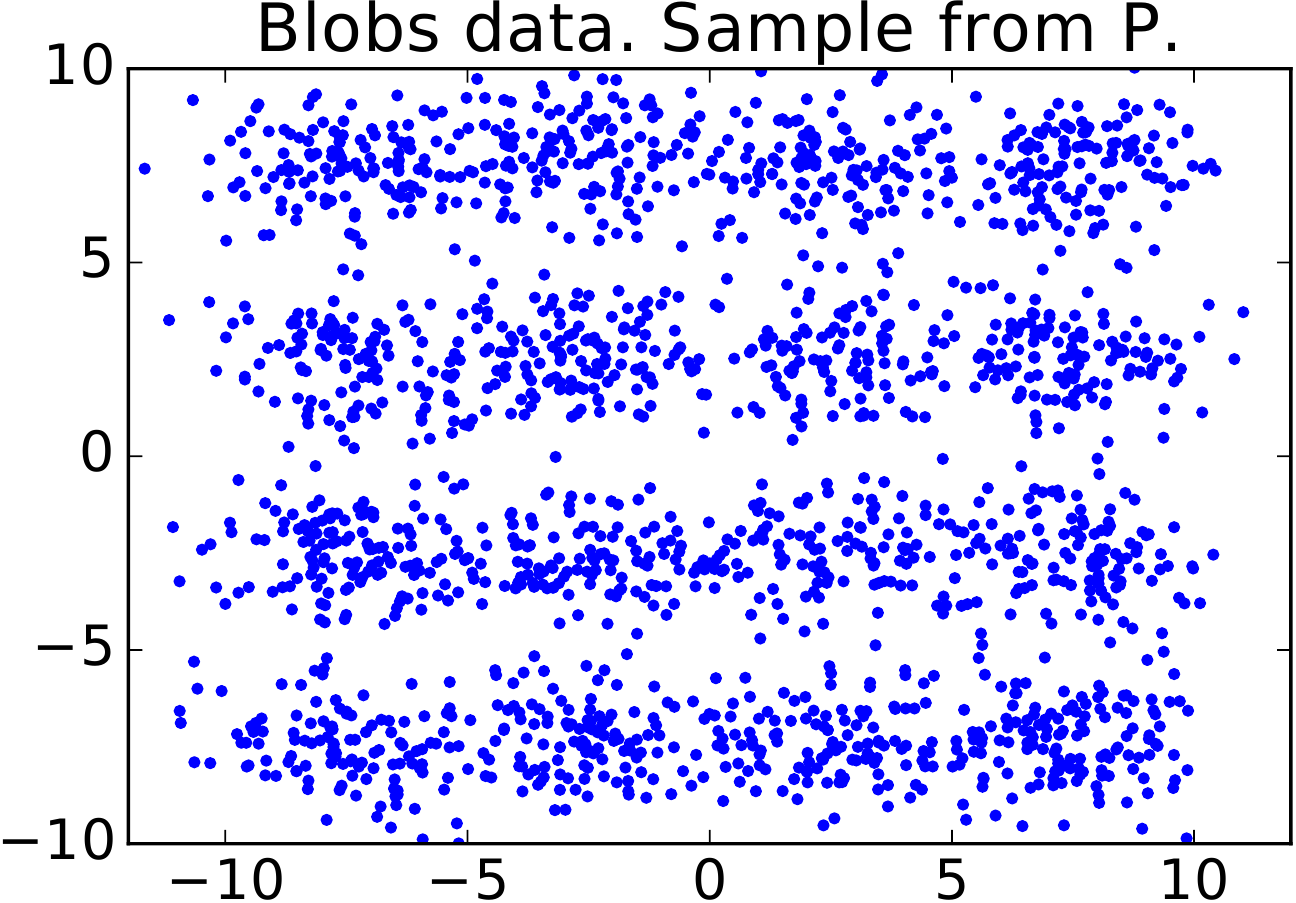} \includegraphics[width=30mm]{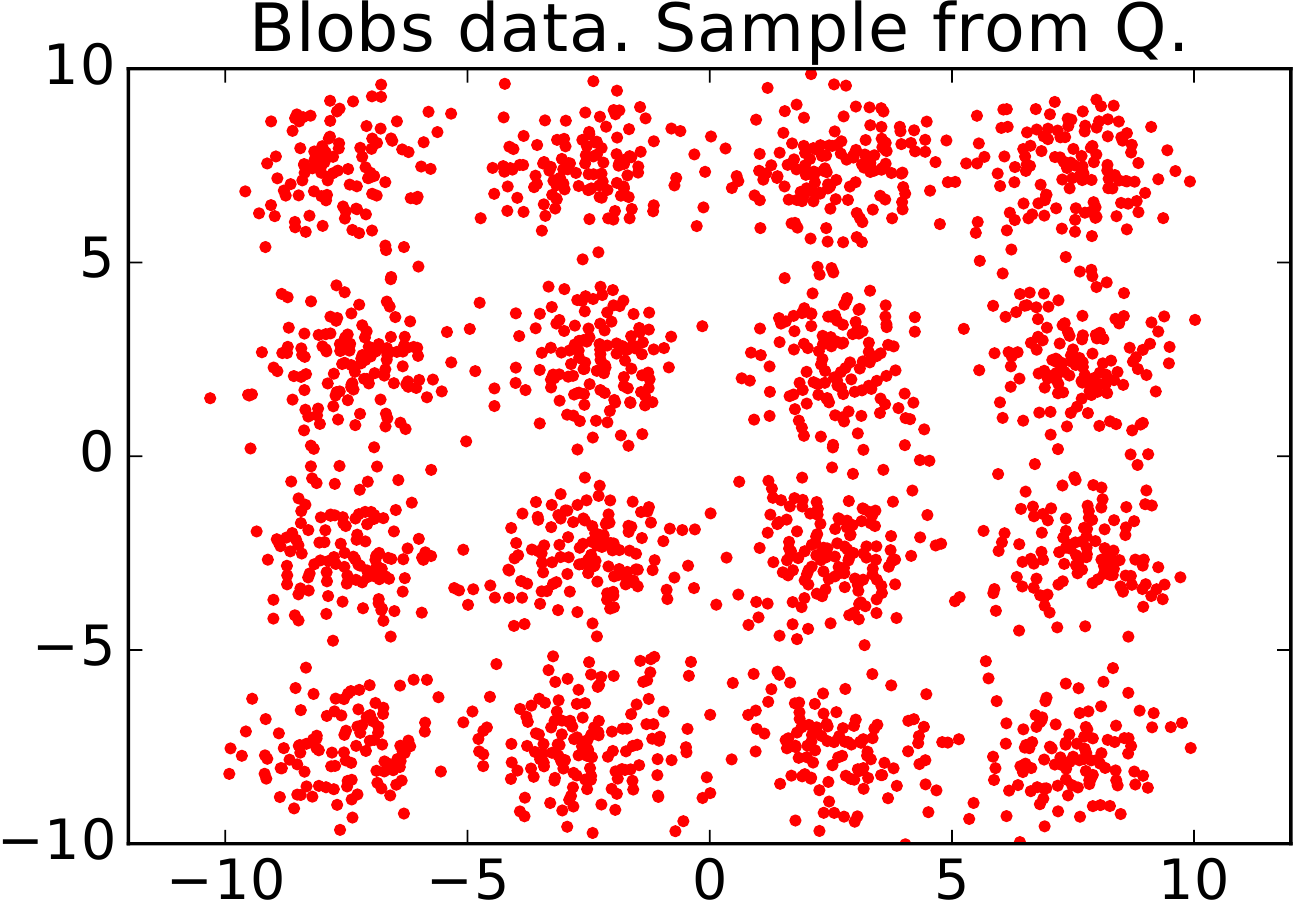} }  \tabularnewline
\bottomrule
\end{tabular} 
\vspace{-4mm}
\end{wraptable} 

In this section, we demonstrate the effectiveness of the proposed
methods on both toy and real problems. We consider the isotropic Gaussian
kernel class $\mathcal{K}_{g}$ in all kernel-based tests. We study
seven two-sample test algorithms. For the SCF test, we set $\hat{l}(\mathbf{x})=k(\mathbf{x},\mathbf{0})$.
Denote by ME-full and SCF-full the ME and SCF tests whose test locations
and the Gaussian width $\sigma$ are fully optimized using gradient
ascent on a separate training sample $(\mathsf{D}^{tr})$ of the same
size as the test set $(\mathsf{D}^{te})$. ME-grid and SCF-grid are
as in \citet{Chwialkowski2015} where only the Gaussian width is optimized
by a grid search,\footnote{\citet{Chwialkowski2015} chooses the Gaussian width that minimizes
the median of the p-values, a heuristic that does not directly address
test power. Here, we perform a grid search to choose the best Gaussian
width by maximizing $\hat{\lambda}_{n/2}^{tr}$ as done in ME-full
and SCF-full.}and the test locations are randomly drawn from a multivariate normal
distribution. MMD-quad (quadratic-time) and MMD-lin (linear-time)
refer to the nonparametric tests based on maximum mean discrepancy
of \citet{Gretton2012}, where to ensure a fair comparison, the Gaussian
kernel width is also chosen so as to maximize a criterion for the
test power on training data, following the same principle as \citep{Gretton2012a}.
For MMD-quad, since its null distribution is given by an infinite
sum of weighted chi-squared variables (no closed-form quantiles),
in each trial we randomly permute the two samples 400 times to approximate
the null distribution. %
Finally, $T^{2}$ is the standard two-sample Hotelling's T-squared
test, which serves as a baseline with Gaussian assumptions on $P$
and $Q$. 

In all the following experiments, each problem is repeated for 500
trials. For toy problems, new samples are generated from the specified
$P,Q$ distributions in each trial. For real problems, samples are
partitioned randomly into training and test sets in each trial. In
all of the simulations, we report an empirical estimate of $\mathbb{P}(\hat{\lambda}_{n/2}^{te}\ge T_{\alpha})$
which is the proportion of the number of times the statistic $\hat{\lambda}_{n/2}^{te}$
is above $T_{\alpha}$. This quantity is an estimate of type-I error
under $H_{0}$, and corresponds to test power when $H_{1}$ is true.
We set $\alpha=0.01$ in all the experiments. All the code and preprocessed
data are available at \url{https://github.com/wittawatj/interpretable-test}.

\textbf{Optimization\,} The parameter tuning objective $\hat{\lambda}_{n/2}^{tr}(\theta)$
is a function of $\theta$ consisting of one real-valued $\sigma$
and $J$ test locations each of $d$ dimensions. The parameters $\theta$
can thus be regarded as a $Jd+1$ Euclidean vector. We take the derivative
of $\hat{\lambda}_{n/2}^{tr}(\theta)$ with respect to $\theta$,
and use gradient ascent to maximize it. $J$ is pre-specified and
fixed. For the ME test, we initialize the test locations with realizations
from two multivariate normal distributions fitted to samples from
$P$ and $Q$; this ensures that the initial locations are well supported
by the data. For the SCF test, initialization using the standard normal
distribution is found to be sufficient. The parameter $\gamma_{n}$
is not optimized; we set the regularization parameter $\gamma_{n}$
to be as small as possible while being large enough to ensure that
$(\mathbf{S}_{n}+\gamma_{n}I)^{-1}$ can be stably computed. We emphasize
that both the optimization and testing are linear in $n$. The testing
cost $\mathcal{O}(J^{3}+J^{2}n+dJn)$ and the optimization costs $\mathcal{O}(J^{3}+dJ^{2}n)$
per gradient ascent iteration. Runtimes of all methods are reported
in Sec. \ref{sec:runtimes} in the appendix.

\paragraph{1. Informative features: simple demonstration\label{subsec:Illustrative_experiment}}

\vspace{-2mm}We begin with a demonstration that the proxy $\hat{\lambda}_{n/2}^{tr}(\theta)$
for the test power is informative for revealing the difference of
the two samples in the ME test. We consider the Gaussian Mean Difference
(GMD) problem (see Table\,\ref{tab:toy_problems}), where both $P$
and $Q$ are two-dimensional normal distributions with the difference
in means. We use $J=2$ test locations $\mathbf{v}_{1}$ and $\mathbf{v}_{2}$,
where $\mathbf{v}_{1}$ is fixed to the location indicated by the
black triangle in Fig.\,\ref{fig:lambda_contour}. The contour plot
shows $\mathbf{v}_{2}\mapsto\hat{\lambda}_{n/2}^{tr}(\mathbf{v}_{1},\mathbf{v}_{2})$.

Fig.\,\ref{fig:lambda_contour} (top) suggests that $\hat{\lambda}_{n/2}^{tr}$
is maximized when $\mathbf{v}_{2}$ is placed in either of the two
regions that captures the difference of the two samples i.e., the
region in which the probability masses of $P$ and $Q$ have less
overlap. Fig.\,\ref{fig:lambda_contour} (bottom), we consider placing
$\mathbf{v}_{1}$ in one of the two key regions. In this case, the
contour plot shows that $\mathbf{v}_{2}$ should be placed in the
other region to maximize $\hat{\lambda}_{n/2}^{tr}$, implying that
placing multiple test locations in the same neighborhood will not
increase the discriminability. The two modes on the left and right
suggest two ways to place the test location in a region that reveals
the difference. The non-convexity of the $\hat{\lambda}_{n/2}^{tr}$
is an indication of many informative ways to detect differences of
$P$ and $Q$, rather than a drawback. A convex objective would not
capture this multimodality. 

\begin{wrapfigure}{r}{0.25\textwidth}   
\vspace{-8mm}
\begin{center}     
\includegraphics[width=3.5cm]{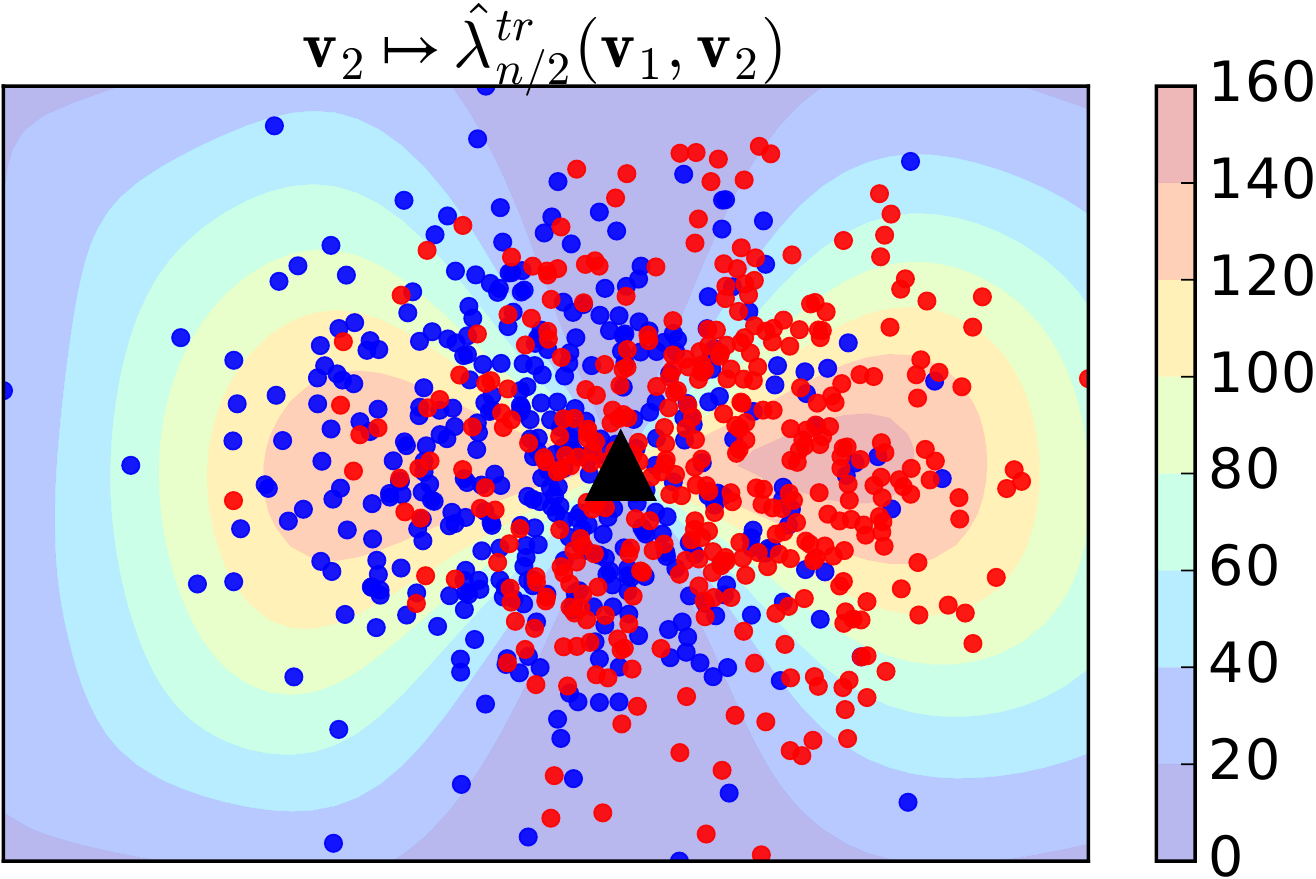}
\includegraphics[width=3.5cm]{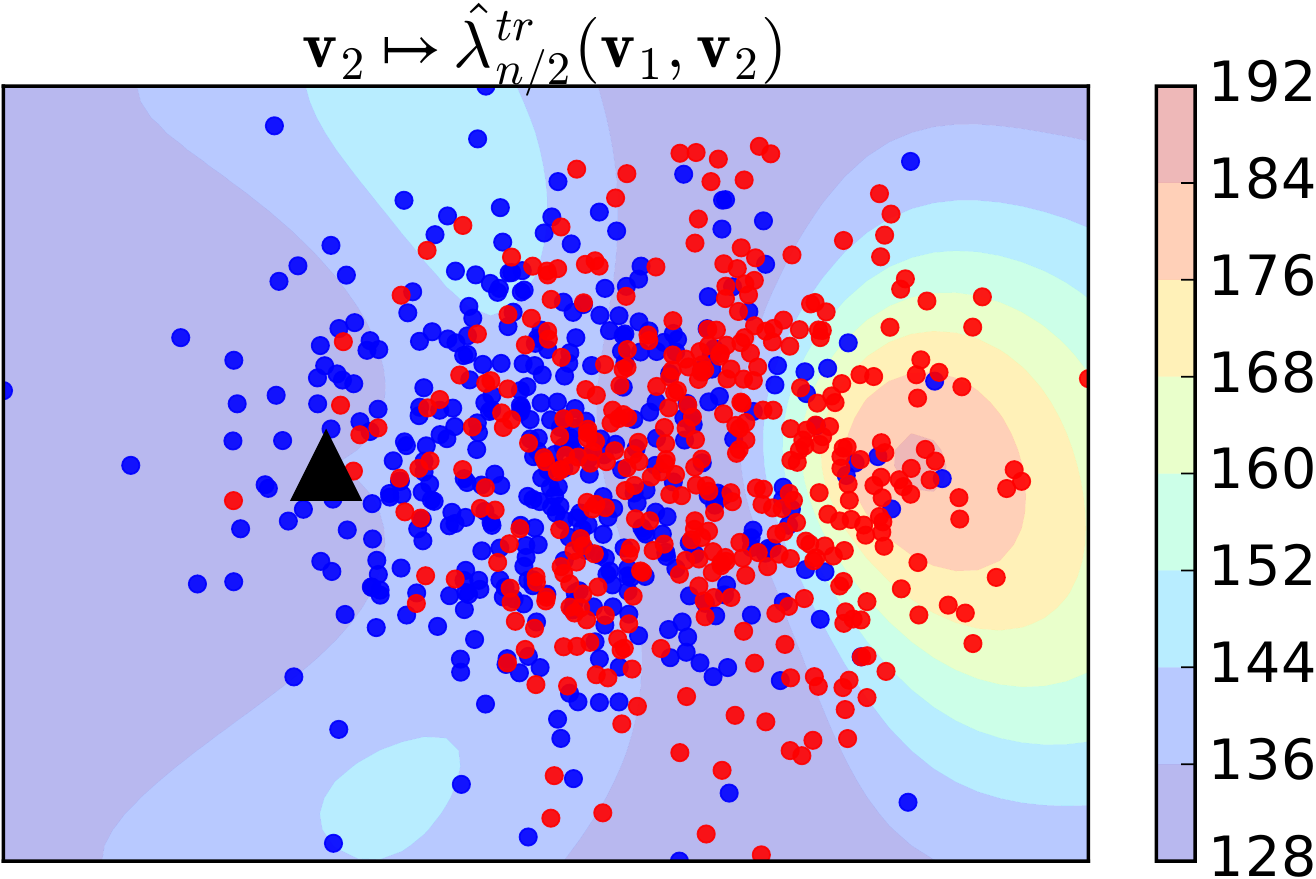}
\end{center}   
\caption{A contour plot of $\hat{\lambda}^{tr}_{n/2}$ as a function of $\mathbf{v}_2$  when $J=2$ and $\mathbf{v}_1$ is fixed (black triangle). The objective $\hat{\lambda}^{tr}_{n/2}$ is high in the regions that reveal the difference of the two samples.} 
\label{fig:lambda_contour}
\vspace{-6mm}
\end{wrapfigure} 

\paragraph{2. Test power vs. sample size $n$ \label{subsec:pow_vs_n}}

\vspace{-2mm}

\begin{figure}
\vspace{-3mm}
\hspace{-2mm}
\subfloat[SG. $d=50$. \label{fig:ex1_sg}]{
\includegraphics[width=0.23\linewidth]{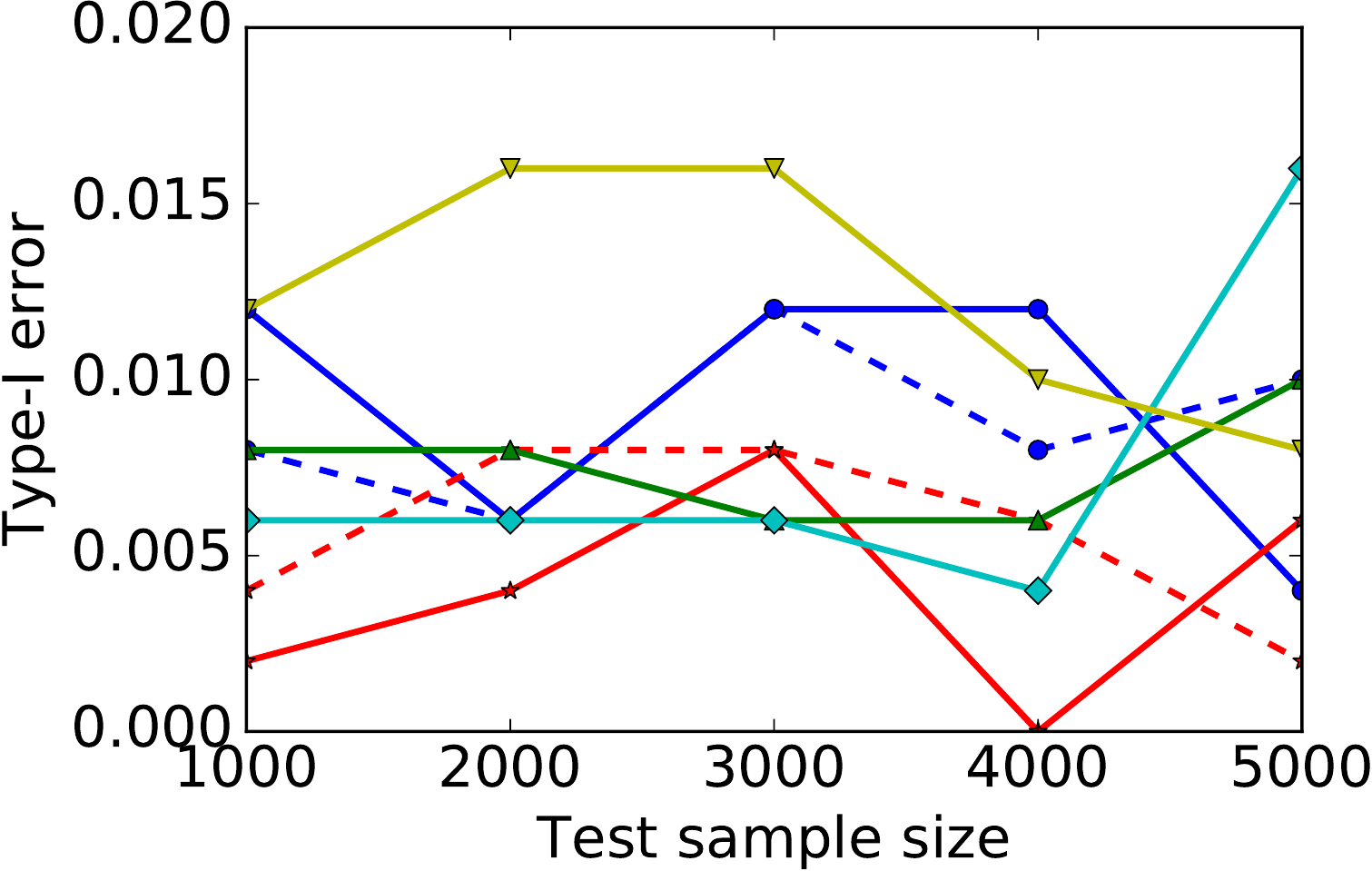}
}\hspace{-2mm}
\subfloat[GMD. $d=100$. \label{fig:ex1_gmd}]{ \includegraphics[width=0.22\linewidth]{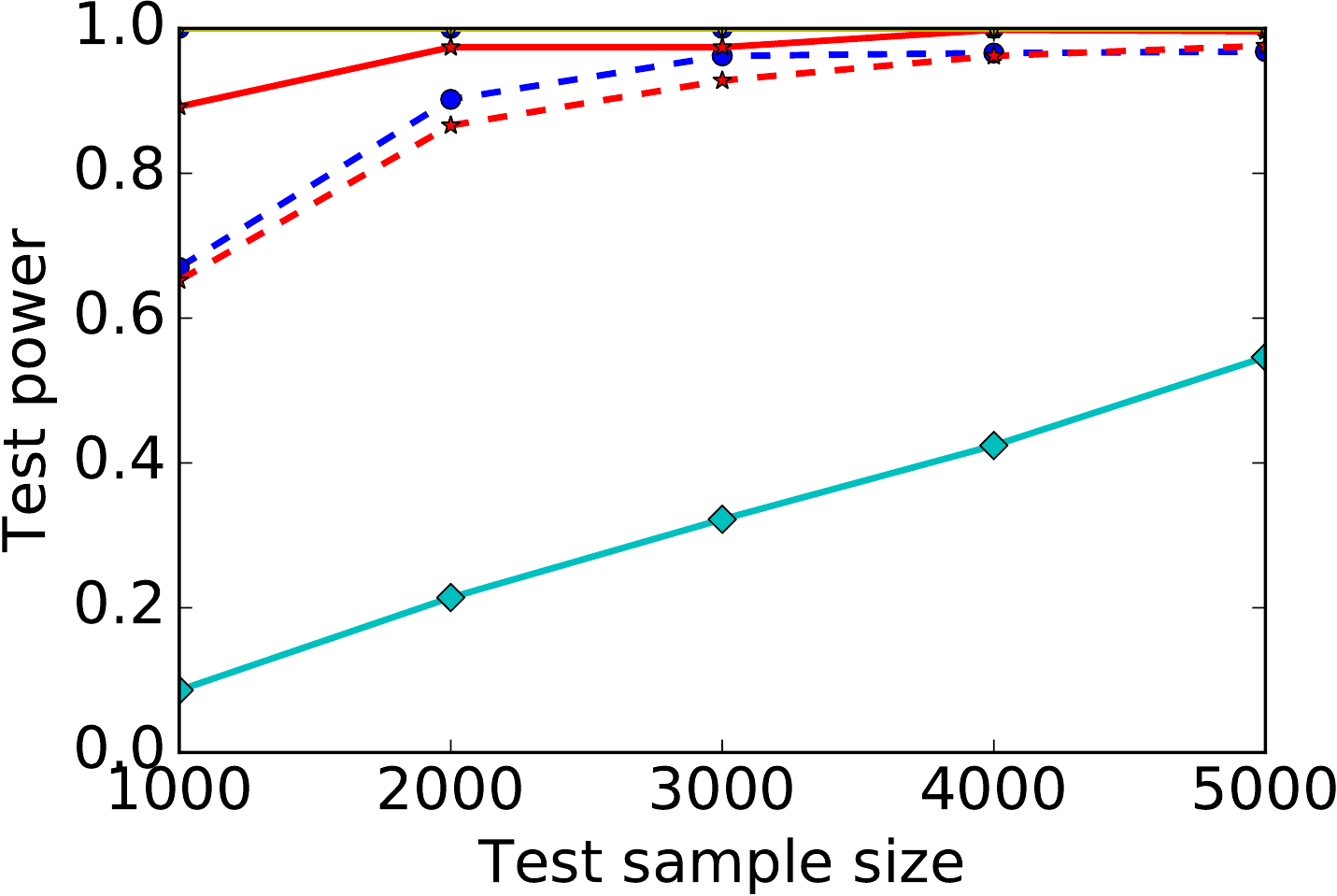} 
}\hspace{-2mm}
\subfloat[GVD. $d=50$. \label{fig:ex1_gvd}]{ \includegraphics[width=0.22\linewidth]{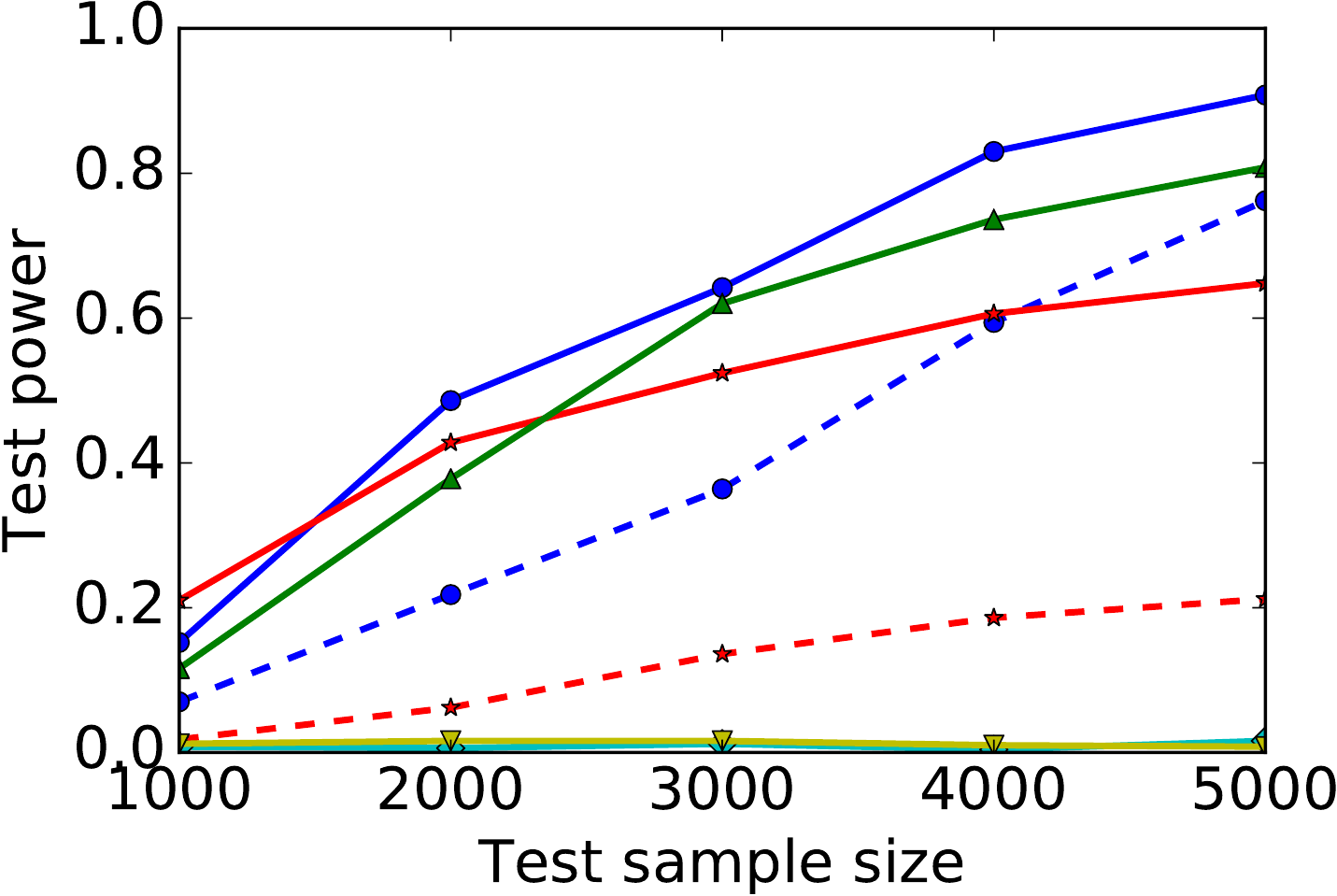} 
} \hspace{-2mm}
\subfloat[Blobs. \label{fig:ex1_blobs}]{ \includegraphics[width=0.31\linewidth]{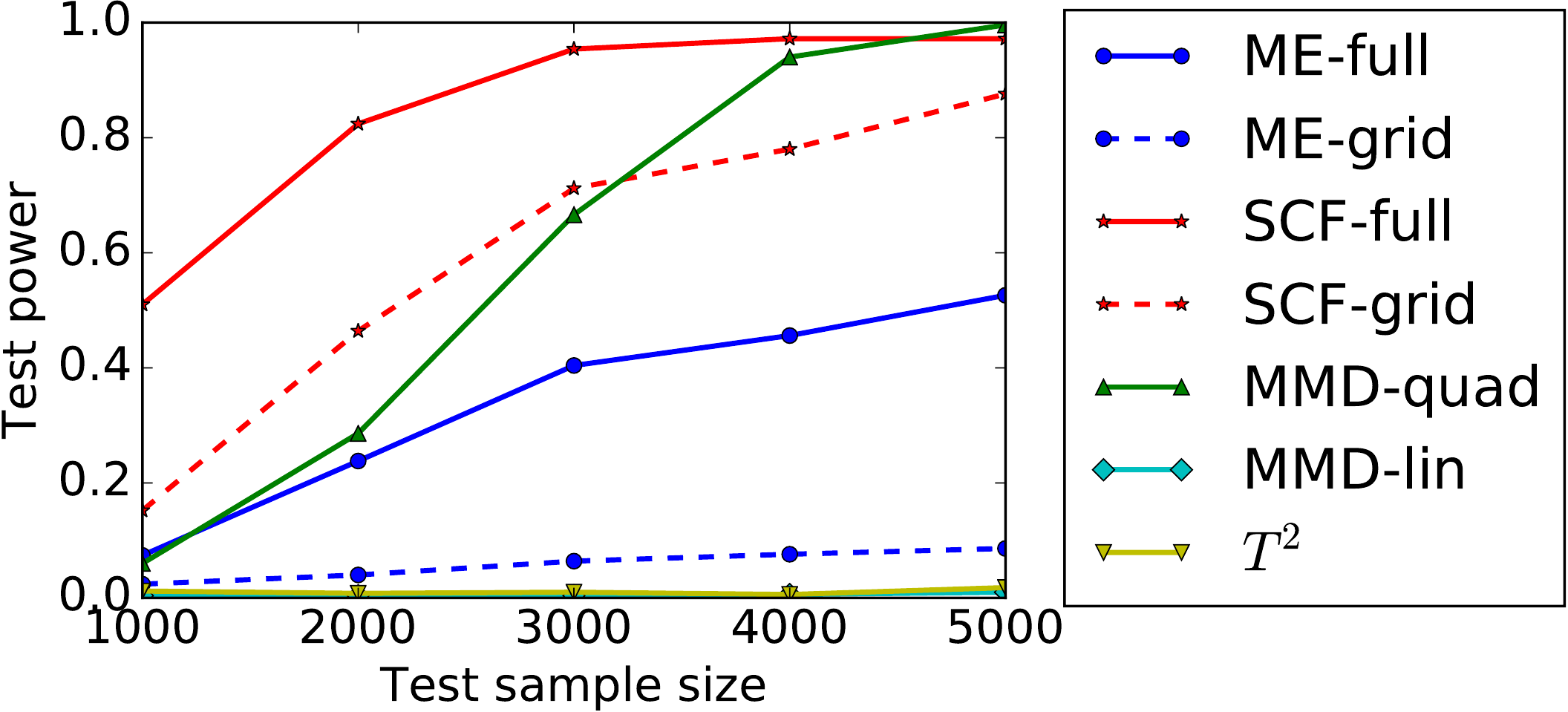} 
}

\caption{Plots of type-I error/test power against the test sample size $n^{te}$
in the four toy problems. \label{fig:pow_vs_n}}
\end{figure}

We now demonstrate the rate of increase of test power with sample
size. When the null hypothesis holds, the type-I error stays at the
specified level $\alpha$. We consider the following four toy problems:
Same Gaussian (SG), Gaussian mean difference (GMD), Gaussian variance
difference (GVD), and Blobs. The specifications of $P$ and $Q$ are
summarized in Table.\,\ref{tab:toy_problems}. In the Blobs problem,
$P$ and $Q$ are defined as a mixture of Gaussian distributions arranged
on a $4\times4$ grid in $\mathbb{R}^{2}$. This problem is challenging
as the difference of $P$ and $Q$ is encoded at a much smaller length
scale compared to the global structure \citep{Gretton2012a}. Specifically,
the eigenvalue ratio for the covariance of each Gaussian distribution
is $2.0$ in $P$, and $1.0$ in $Q$. We set $J=5$ in this experiment. 

The results are shown in Fig.\,\ref{fig:pow_vs_n} where type-I error
(for SG problem), and test power (for GMD, GVD and Blobs problems)
are plotted against test sample size. A number of observations are
worth noting. In the SG problem, we see that the type-I error roughly
stays at the specified level: the rate of rejection of $H_{0}$ when
it is true is roughly at the specified level $\alpha=0.01$.

GMD with 100 dimensions turns out to be an easy problem for all the
tests except MMD-lin. In the GVD and Blobs cases, ME-full and SCF-full
achieve substantially higher test power than ME-grid and SCF-grid,
respectively, suggesting a clear advantage from optimizing the test
locations. Remarkably, ME-full consistently outperforms the quadratic-time
MMD across all test sample sizes in the GVD case. When the difference
of $P$ and $Q$ is subtle as in the Blobs problem, ME-grid, which
uses randomly drawn test locations, can perform poorly (see Fig.\,\ref{fig:ex1_blobs})
since it is unlikely that randomly drawn locations will be placed
in the key regions that reveal the difference. In this case, optimization
of the test locations can considerably boost the test power (see ME-full
in Fig.\,\ref{fig:ex1_blobs}). Note also that SCF variants perform
significantly better than ME variants on the Blobs problem, as the
difference in $P$ and $Q$ is localized in the frequency domain;
ME-full and ME-grid would require many more test locations in the
spatial domain to match the test powers of the SCF variants. For the
same reason, SCF-full does much better than the quadratic-time MMD
across most sample sizes, as the latter represents a weighted distance
between characteristic functions integrated across the entire frequency
domain \citep[Corollary 4]{SriGreFukSchetal10}.

\paragraph{3. Test power vs. dimension $d$\label{subsec:pow_vs_d}}

\vspace{-1mm}
\begin{figure*}
\vspace{-3mm}
\centering
\subfloat[SG  \label{fig:ex2_sg}]{
\includegraphics[width=0.25\textwidth]{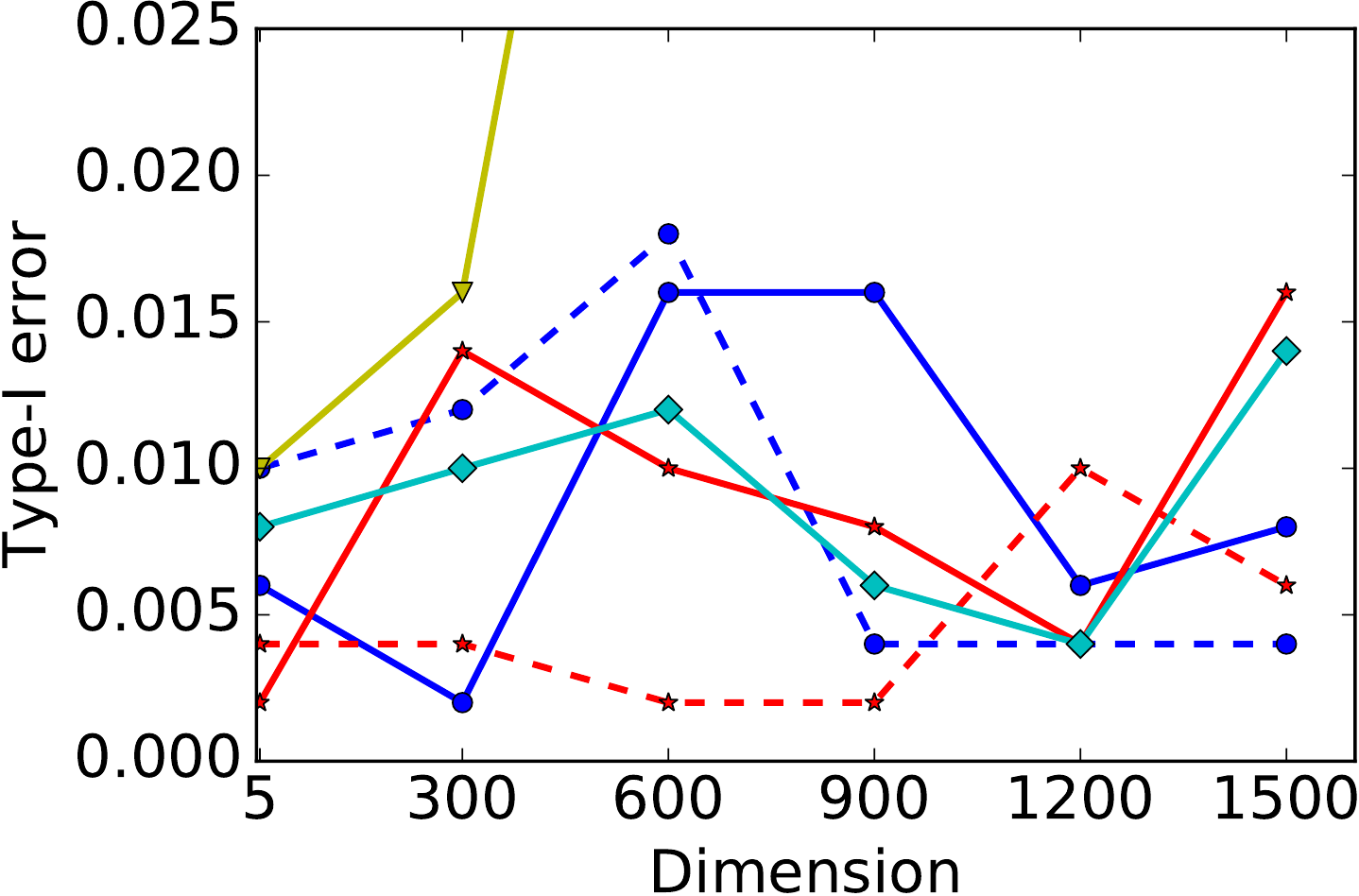}
}
\subfloat[GMD \label{fig:ex2_gmd}]{ \includegraphics[width=0.24\textwidth]{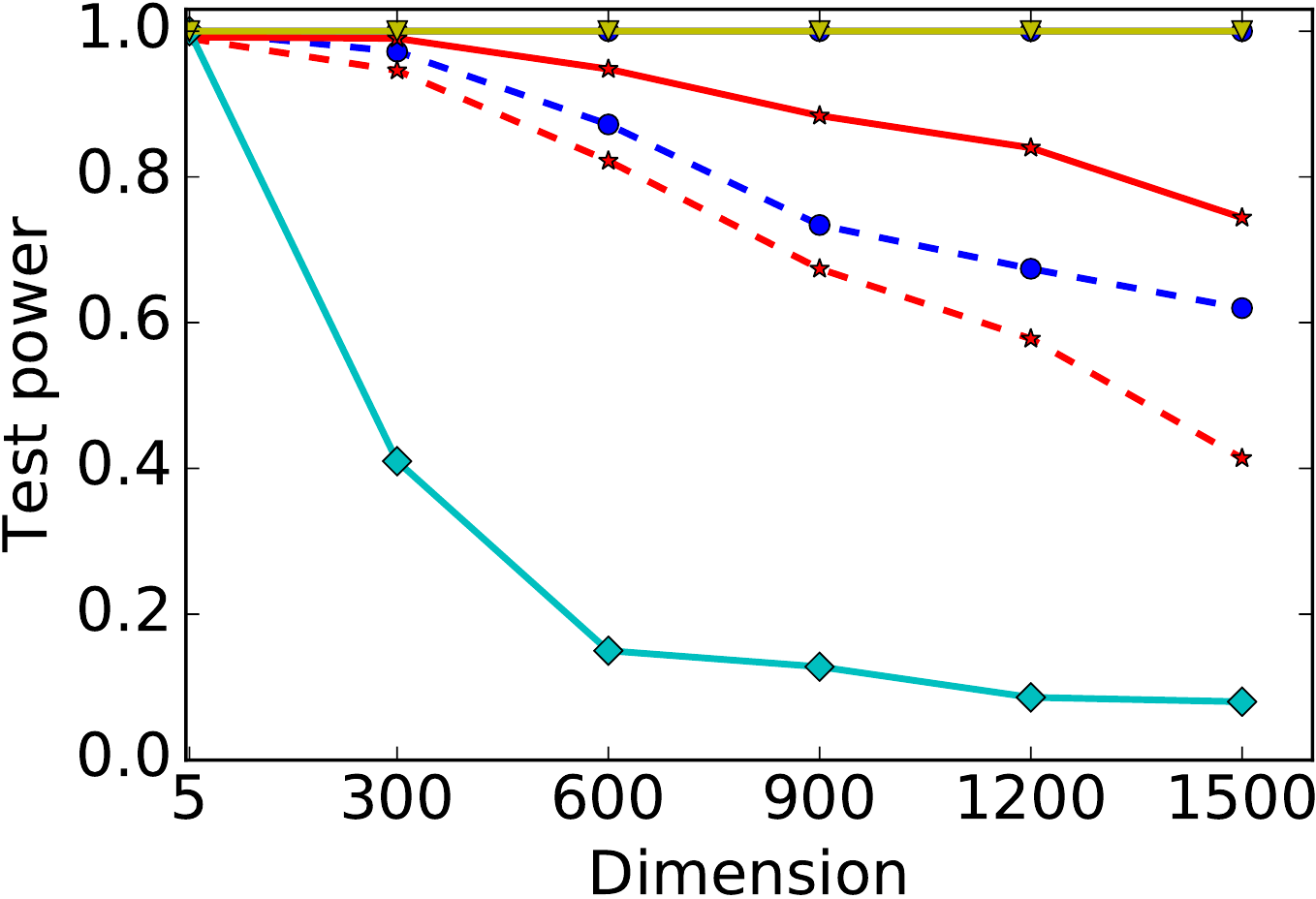} 
}
\subfloat[GVD \label{fig:ex2_gvd}]{ \includegraphics[width=0.36\textwidth]{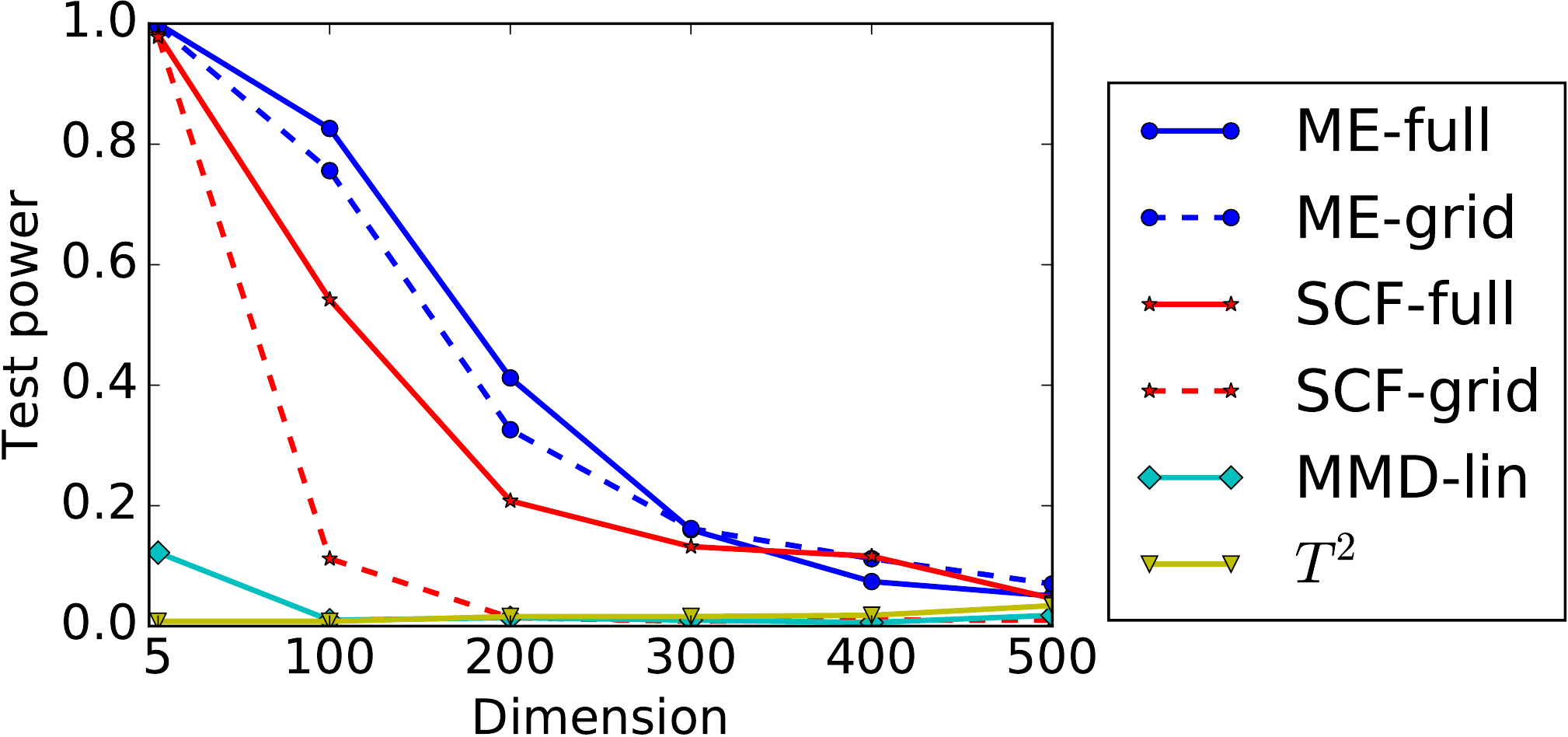} 
} 

\caption{Plots of type-I error/test power against the dimensions $d$ in the
four toy problems in Table\,\ref{tab:toy_problems}.\label{fig:power_vs_d}}
\end{figure*}

We next investigate how the dimension ($d$) of the problem can affect
type-I errors and test powers of ME and SCF tests. We consider the
same artificial problems: SG, GMD and GVD. This time, we fix the test
sample size to 10000, set $J=5$, and vary the dimension. The results
are shown in Fig.\,\ref{fig:power_vs_d}. Due to the large dimensions
and sample size, it is computationally infeasible to run MMD-quad.

We observe that all the tests except the T-test can maintain type-I
error at roughly the specified significance level $\alpha=0.01$ as
dimension increases. The type-I performance of the T-test is incorrect
at large $d$ because of the difficulty in accurately estimating the
covariance matrix in high dimensions. %
It is interesting to note the high performance of ME-full in the GMD
problem in Fig.\,\ref{fig:ex2_gmd}. ME-full achieves the maximum
test power of 1.0 throughout and matches the power T-test, in spite
of being nonparametric and making no assumption on $P$ and $Q$ (the
T-test is further advantaged by its excessive Type-I error). However,
this is true only with optimization of the test locations. This is
reflected in the test power of ME-grid in Fig.\,\ref{fig:ex2_gmd}
which drops monotonically as dimension increases, highlighting the
importance of test location optimization. The performance of MMD-lin
degrades quickly with increasing dimension, as expected from \citet{Ramdas2015}. 

\begin{table}
\caption{Type-I errors and powers of various tests in the problem of distinguishing
NIPS papers from two categories. $\alpha=0.01$. $J=1$. $n_{te}$
denotes the test sample size of each of the two samples.\label{tab:nips_power}}

\centering
\begin{tabular}{lc|cccccc}
\toprule
Problem & $n^{te}$ & ME-full & ME-grid & SCF-full & SCF-grid & MMD-quad & MMD-lin\\
\midrule
Bayes-Bayes & 215 & .012 & .018 & .012 & .004 & .022 & .008  \\ 
Bayes-Deep & 216 & .954 & .034 & .688 & .180 & .906 & .262  \\ 
Bayes-Learn & 138 & .990 & .774 & .836 & .534 & 1.00 & .238  \\ 
Bayes-Neuro & 394 & 1.00 & .300 & .828 & .500 & .952 & .972  \\ 
Learn-Deep & 149 & .956 & .052 & .656 & .138 & .876 & .500  \\ 
Learn-Neuro & 146 & .960 & .572 & .590 & .360 & 1.00 & .538  \\ 
\bottomrule
\end{tabular}
\vspace{-3mm}
\end{table}

\paragraph{4. Distinguishing articles from two categories\label{subsec:nips_experiment}}

\vspace{-1mm}We now turn to performance on real data. We first consider
the problem of distinguishing two categories of publications at the
conference on Neural Information Processing Systems (NIPS). Out of
5903 papers published in NIPS from 1988 to 2015, we manually select
disjoint subsets related to Bayesian inference (Bayes), neuroscience
(Neuro), deep learning (Deep), and statistical learning theory (Learn)
(see Sec.\,\ref{sec:NIPS_appendix}). Each paper is represented as
a bag of words using TF-IDF \citep{Manning2008} as features. We perform
stemming, remove all stop words, and retain only nouns. A further
filtering of document-frequency (DF) of words that satisfies $5\le\mathrm{DF}\leq2000$
yields approximately 5000 words from which 2000 words (i.e., $d=2000$
dimensions) are randomly selected. See Sec.\,\ref{sec:NIPS_appendix}
for more details on the preprocessing. For ME and SCF tests, we use
only one test location i.e., set $J=1$. We perform 1000 permutations
to approximate the null distribution of MMD-quad in this and the following
experiments.

Type-I errors and test powers are summarized in Table.\,\ref{tab:nips_power}.
The first column indicates the categories of the papers in the two
samples. In Bayes-Bayes problem, papers on Bayesian inference are
randomly partitioned into two samples in each trial. This task represents
a case in which $H_{0}$ holds. Among all the linear-time tests, we
observe that ME-full has the highest test power in all the tasks,
attaining a maximum test power of 1.0 in the Bayes-Neuro problem.
This high performance assures that although different test locations
$\mathcal{V}$ may be selected in different trials, these locations
are each informative. It is interesting to observe that ME-full has
performance close to or better than MMD-quad, which requires $O(n^{2})$
runtime complexity. Besides clear advantages of interpretability and
linear runtime of the proposed tests, these results suggest that evaluating
the differences in expectations of analytic functions at particular
locations can yield an equally powerful test at a much lower cost,
as opposed to computing the RKHS norm of the witness function as done
in MMD. Unlike Blobs, however, Fourier features are less powerful
in this setting.

We further investigate the interpretability of the ME test by the
following procedure. For the learned test location $\mathbf{v}^{t}\in\mathbb{R}^{d}$
($d=2000$) in trial $t$, we construct $\tilde{\mathbf{v}}^{t}=\left(\tilde{v}_{1}^{t},\ldots,\tilde{v}_{d}^{t}\right)$
such that $\tilde{v}_{j}^{t}=|v_{j}^{t}|$. Let $\eta_{j}^{t}\in\{0,1\}$
be an indicator variable taking value 1 if $\tilde{v}_{j}^{t}$ is
among the top five largest for all $j\in\{1,\ldots,d\}$, and 0 otherwise.
Define $\eta_{j}:=\sum_{t}\eta_{j}^{t}$ as a proxy indicating the
significance of word $j$ i.e., $\eta_{j}$ is high if word $j$ is
frequently among the top five largest as measured by $\tilde{v}_{j}^{t}$.
The top seven words as sorted in descending order by $\eta_{j}$ in
the Bayes-Neuro problem are \emph{spike, markov, cortex, dropout,
recurr, iii, gibb}, showing that the learned test locations are highly
interpretable. Indeed, ``markov'' and ``gibb'' (i.e., stemmed
from Gibbs) are discriminative terms in Bayesian inference category,
and ``spike'' and ``cortex'' are key terms in neuroscience. We
give full lists of discriminative terms learned in all the problems
in Sec.\,\ref{subsec:nips_exp_terms}. To show that not all the randomly
selected 2000 terms are informative, if the definition of $\eta_{j}^{t}$
is modified to consider the least important words (i.e., $\eta_{j}$
is high if word $j$ is frequently among the top five smallest as
measured by $\tilde{v}_{j}^{t}$), we instead obtain \emph{circumfer,
bra, dominiqu, rhino, mitra, kid, impostor,} which are not discriminative.

\paragraph{5. Distinguishing positive and negative emotions\label{subsec:face_experiment}}

\newcommand{\facewidth}{10mm}
\begin{wrapfigure}{r}{0.34\textwidth}   
\vspace{-8mm}
\begin{center}     
\subfloat[HA \label{fig:kdef_ha}]{ 
\includegraphics[width=\facewidth]{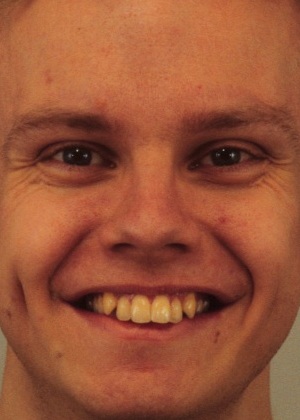} 
}
\subfloat[NE]{ 
\includegraphics[width=\facewidth]{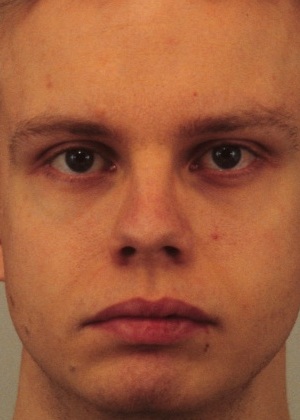} 
}
\subfloat[SU]{ 
\includegraphics[width=\facewidth]{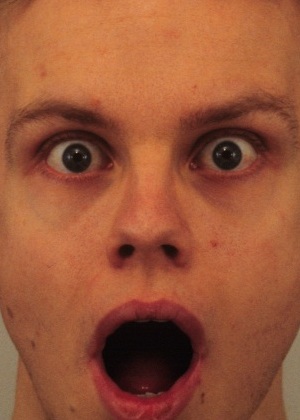} 
}

\subfloat[AF]{ 
\includegraphics[width=\facewidth]{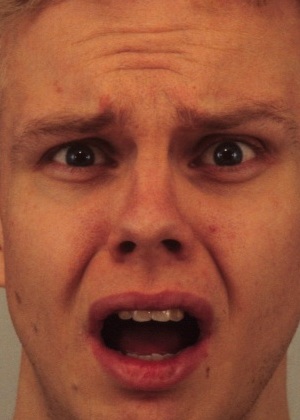} 
}
\subfloat[AN]{ 
\includegraphics[width=\facewidth]{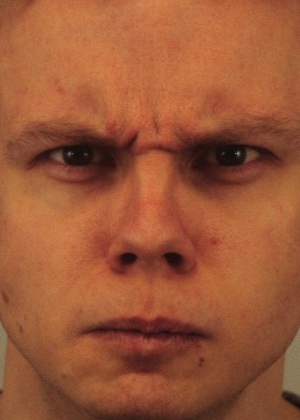} 
}
\subfloat[DI\label{fig:kdef_di}]{ 
\includegraphics[width=\facewidth]{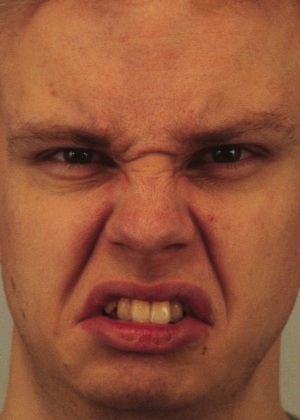} 
}
\subfloat[$\mathbf{v}_1$ \label{fig:v1_face_diff}]{
\includegraphics[width=1cm]{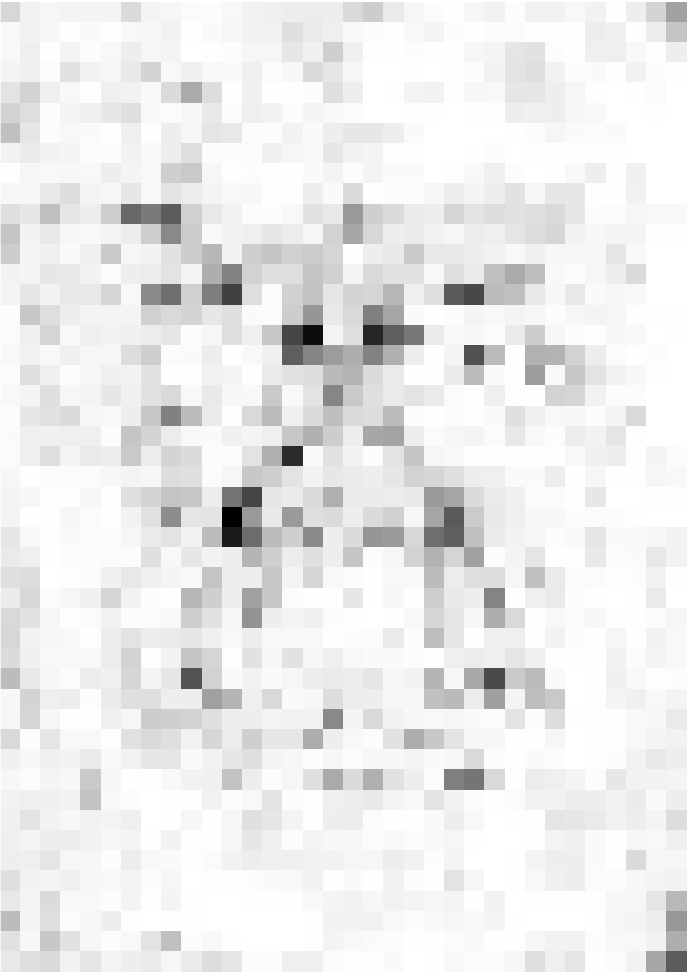}
}
\end{center}   
\caption{(a)-(f): Six facial expressions of actor AM05 in the KDEF data. (g): Average across trials of the learned test locations $\mathbf{v}_1$.} 
\label{fig:face_samples}
\end{wrapfigure} 

In the final experiment, we study how well ME and SCF tests can distinguish
two samples of photos of people showing positive and negative facial
expressions. Our emphasis is on the discriminative features of the
faces identified by ME test showing how the two groups differ. For
this purpose, we use Karolinska Directed Emotional Faces (KDEF) dataset
\citep{Lundqvist1998} containing 5040 aligned face images of 70 amateur
actors, 35 females and 35 males. We use only photos showing front
views of the faces. In the dataset, each actor displays seven expressions:
happy (HA), neutral (NE), surprised (SU), sad (SA), afraid (AF), angry
(AN), and disgusted (DI). We assign HA, NE, and SU faces into the
positive emotion group (i.e., samples from $P$), and AF, AN and DI
faces into the negative emotion group (samples from $Q$). We denote
this problem as ``$+$ vs.\ $-$''. Examples of six facial expressions
from one actor are shown in Fig.\,\ref{fig:face_samples}. Photos
of the SA group are unused to keep the sizes of the two samples the
same. Each image of size $562\times762$ pixels is cropped to exclude
the background, resized to $48\times34=1632$ pixels ($d$), and converted
to grayscale. 

We run the tests 500 times with the same setting used previously i.e.,
Gaussian kernels, and $J=1$. The type-I errors and test powers are
shown in Table\,\ref{tab:face_power}. In the table, ``$\pm$ vs.\
$\pm$'' is a problem in which all faces expressing the six emotions
are randomly split into two samples of equal sizes i.e., $H_{0}$
is true. Both ME-full and SCF-full achieve high test powers while
maintaining the correct type-I errors.

\begin{table}
\caption{Type-I errors and powers in the problem of distinguishing positive
(+) and negative (-) facial expressions. $\alpha=0.01$. $J=1$. \label{tab:face_power}}

\centering
\begin{tabular}{lc|cccccc}
\toprule
Problem & $n^{te}$ & ME-full & ME-grid & SCF-full & SCF-grid & MMD-quad & MMD-lin\\
\midrule
$\pm$ vs. $\pm$ & 201 & .010 & .012 & .014 & .002 & .018 & .008 \\ 
$+$ vs. $-$ & 201 & .998 & .656 & 1.00 & .750 & 1.00 & .578 \\ 
\bottomrule
\end{tabular}
\vspace{-3mm}
\end{table}

As a way to interpret how positive and negative emotions differ, we
take an average across trials of the learned test locations of ME-full
in the ``$+$ vs.\ $-$'' problem. This average is shown in Fig.\,\ref{fig:v1_face_diff}.
We see that the test locations faithfully capture the difference of
positive and negative emotions by giving more weights to the regions
of nose, upper lip, and nasolabial folds (smile lines), confirming
the interpretability of the test in a high-dimensional setting.

\subsubsection*{Acknowledgement}

We thank the Gatsby Charitable Foundation for the financial support. 

{\small \bibliographystyle{myplainnat}
\bibliography{fotest,fotest_Z}
}

\newpage
\appendix
\newgeometry{left=3cm, right=3cm, top=2.5cm, bottom=2.5cm}

\begin{center}
{\Large{}Interpretable Distribution Features with Maximum Testing
Power}
\par\end{center}{\Large \par}

\begin{center}
\textcolor{black}{\Large{}Supplementary Material}
\par\end{center}{\Large \par}

\section{Algorithm\label{sec:Algorithm}}

The full algorithm for the proposed tests from parameter tuning to
the actual two-sample testing is given in Algorithm\,\ref{algo:full}.

\begin{algorithm}[h]
\caption{Optimizing parameters and testing}
\label{algo:full}
\begin{algorithmic}[1]
\REQUIRE Two samples $\mathsf{X}$, $\mathsf{Y}$, significance level $\alpha$, and 
 number of test locations $J$
\STATE Split $\mathsf{D} := (\mathsf{X}, \mathsf{Y})$ into disjoint training and test sets, $\mathsf{D}^{tr}$ and $\mathsf{D}^{te}$, of the same size $n^{te}$.
\STATE Optimize parameters $\theta = \arg\max_\theta \hat{\lambda}_{n/2}^{tr}(\theta)$ where $\hat{\lambda}_{n/2}^{tr}(\theta)$ is computed with the training set $\mathsf{D}^{tr}$.
\STATE Set $T_{\alpha}$ to the $(1-\alpha)$-quantile of $\chi^2(J')$.
\STATE Compute the test statistic $\hat{\lambda}_{n/2}^{te}(\theta)$ using $\mathsf{D}^{te}$. 
\STATE Reject $H_0$ if $\hat{\lambda}_{n/2}^{te}(\theta) > T_{\alpha}$.
\end{algorithmic}
\end{algorithm}

\section{Experiments on NIPS text collection\label{sec:NIPS_appendix}}

The full procedure for processing the NIPS text collection is summarized
as following.
\begin{enumerate}
\item Download all 5903 papers from 1988 to 2015 from \url{https://papers.nips.cc/}
as PDF files.
\item Convert each PDF file to text with \texttt{pdftotext}\footnote{\texttt{pdftotext} is available at \url{http://poppler.freedesktop.org}.}.
\item Remove all stop words. We use the list of stop words from \url{http://www.ranks.nl/stopwords}.
\item Keep only nouns. We use the list of nouns as available in WordNet-3.0\footnote{WordNet is available online at \url{https://wordnet.princeton.edu/wordnet/citing-wordnet/}.}. 
\item Keep only words which contain only English alphabets i.e., does not
contain punctuations or numbers. Also, word length must be between
3 and 20 characters (inclusive).
\item Keep only words which occur in at least 5 documents, and in no more
than 2000 documents.
\item Convert all characters to small case. Stem all words with SnowballStemmer
in NLTK \citep{Bird2009}. For example, ``recognize'' and ``recognizer''
become ``recogn'' after stemming.
\item Categorize papers into disjoint collections. A paper is treated as
belonging to a group if its title has at least one word from the list
of keywords for the category. Papers that match the criteria of both
categories are not considered. The lists of keywords are as follows.

\begin{enumerate}
\item \textbf{Bayesian inference} (Bayes): graphical model, bayesian, inference,
mcmc, monte carlo, posterior, prior, variational, markov, latent,
probabilistic, exponential family.
\item \textbf{Deep learning} (Deep): deep, drop out, auto-encod, convolutional,
neural net, belief net, boltzmann.
\item \textbf{Learning theory} (Learn): learning theory, consistency, theoretical
guarantee, complexity, pac-bayes, pac-learning, generalization, uniform
converg, bound, deviation, inequality, risk min, minimax, structural
risk, VC, rademacher, asymptotic. 
\item \textbf{Neuroscience} (Neuro): motor control, neural, neuron, spiking,
spike, cortex, plasticity, neural decod, neural encod, brain imag,
biolog, perception, cognitive, emotion, synap, neural population,
cortical, firing rate, firing-rate, sensor.
\end{enumerate}
\item Randomly select 2000 words from the remaining words.
\item Treat each paper as a bag of words and construct a feature vector
with TF-IDF \citep{Manning2008}. 
\end{enumerate}

\subsection{Discriminative terms identified by ME test\label{subsec:nips_exp_terms}}

In this section, we provide full lists of discriminative terms following
the procedure described in Sec.\,\ref{subsec:nips_experiment}. The
top ten words in each problem are as follows.
\begin{itemize}
\item \textbf{Bayes-Bayes}: collabor, traffic, bay, permut, net, central,
occlus, mask, draw, joint.
\item \textbf{Bayes-Deep}: infer, bay, mont, adaptor, motif, haplotyp, ecg,
covari, boltzmann, classifi.
\item \textbf{Bayes-Learn}: infer, markov, graphic, segment, bandit, boundari,
favor, carlo, prioriti, prop.
\item \textbf{Bayes-Neuro}: spike, markov, cortex, dropout, recurr, iii,
gibb, basin, circuit, subsystem.
\item \textbf{Learn-Deep}: deep, forward, delay, subgroup, bandit, recept,
invari, overlap, inequ, pia.
\item \textbf{Learn-Neuro}: polici, interconnect, hardwar, decay, histolog,
edg, period, basin, inject, human.
\end{itemize}

\newpage

\section{Runtimes\label{sec:runtimes}}

In this section, we provide runtimes of all the experiments. The runtimes
of the ``Test power vs. sample $n$'' experiment are shown in Fig.
\ref{fig:pow_vs_n_time}. The runtimes of the ``Test power vs. dimension
$d$'' experiment are shown in Fig. \ref{fig:power_vs_d_time}. Table
\ref{tab:nips_time}, \ref{tab:face_time} give the runtimes of the
two real-data experiments.

\begin{figure}[th]
\hspace{-9mm}
\subfloat[SG. $d=50$. \label{fig:ex1_sg_time}]{
\includegraphics[width=0.23\linewidth]{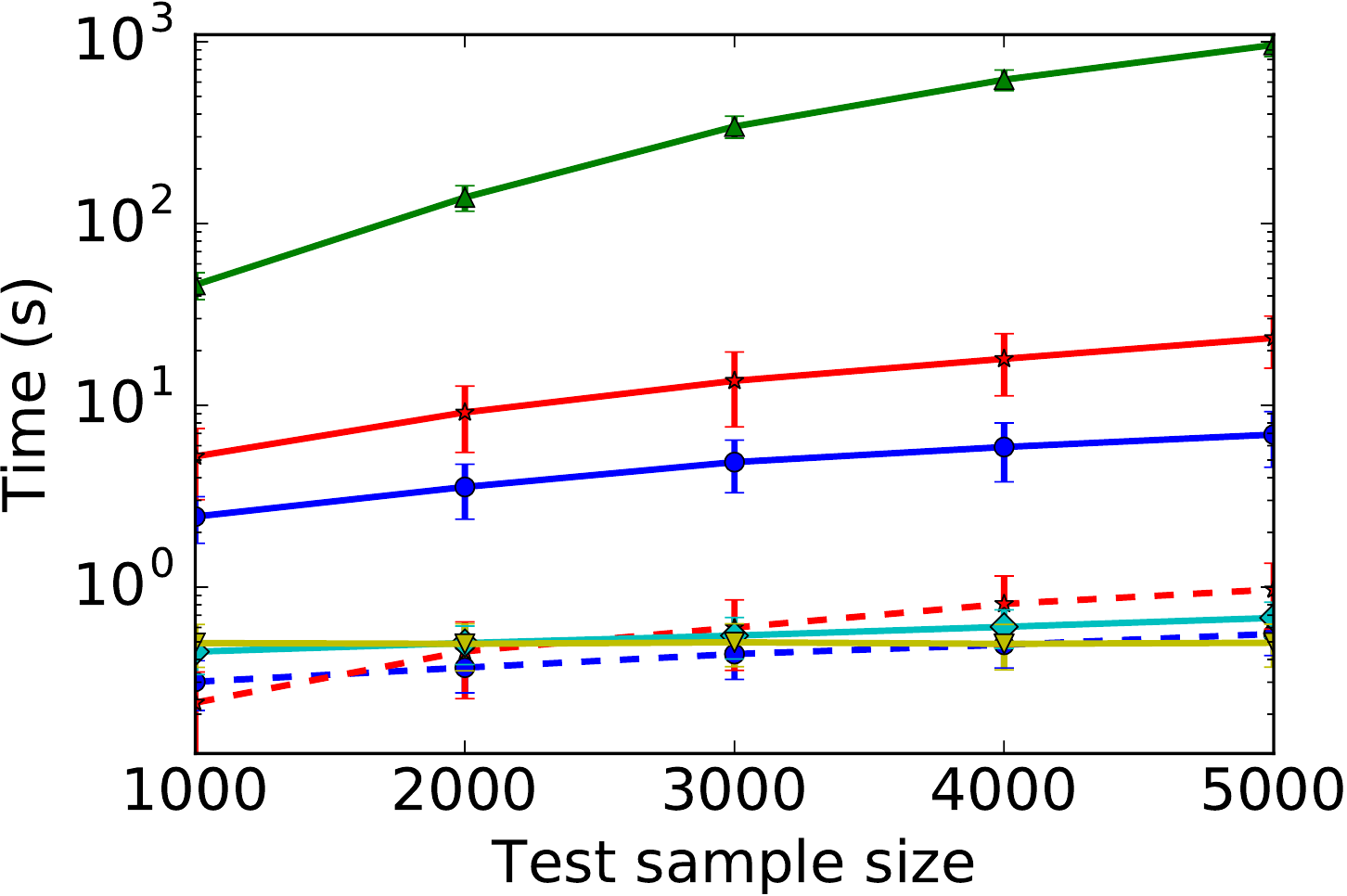}
}\hspace{-2mm}
\subfloat[GMD. $d=100$. \label{fig:ex1_gmd_time}]{ 
\includegraphics[width=0.23\linewidth]{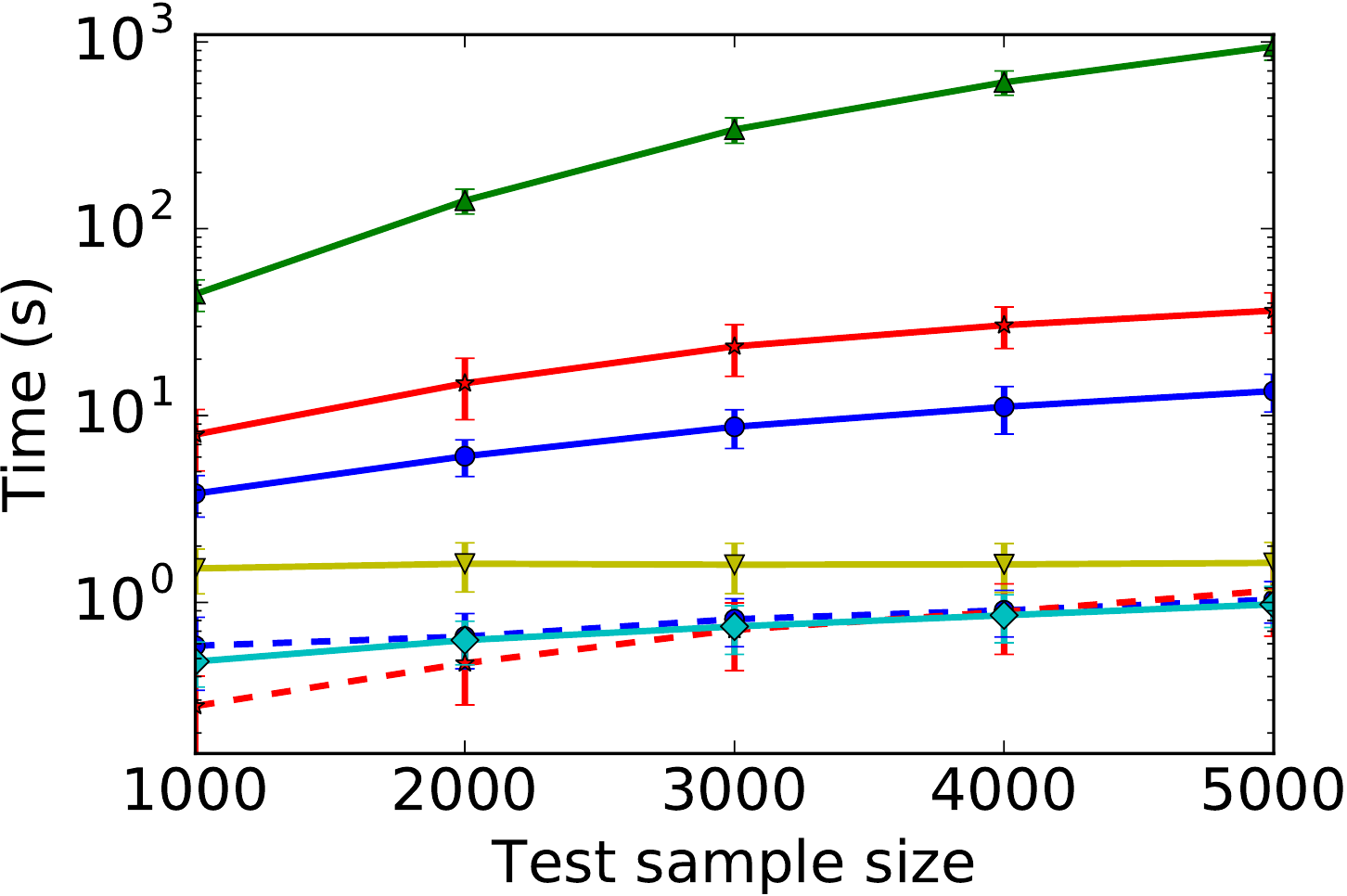} 
}\hspace{-2mm}
\subfloat[GVD. $d=50$. \label{fig:ex1_gvd_time}]{ 
\includegraphics[width=0.23\linewidth]{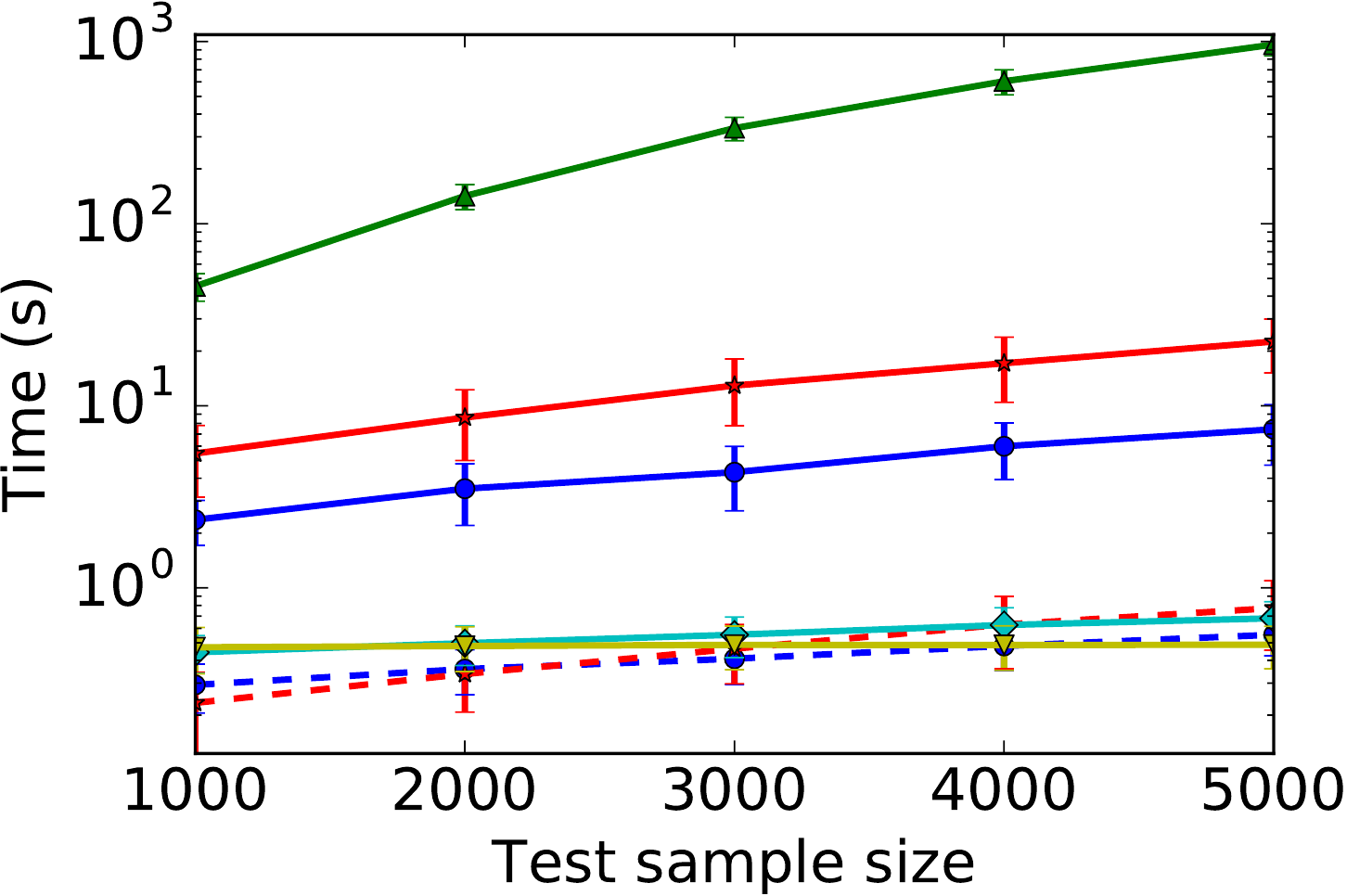} 
} \hspace{-2mm}
\subfloat[Blobs. \label{fig:ex1_blobs_time}]{ 
\includegraphics[width=0.34\linewidth]{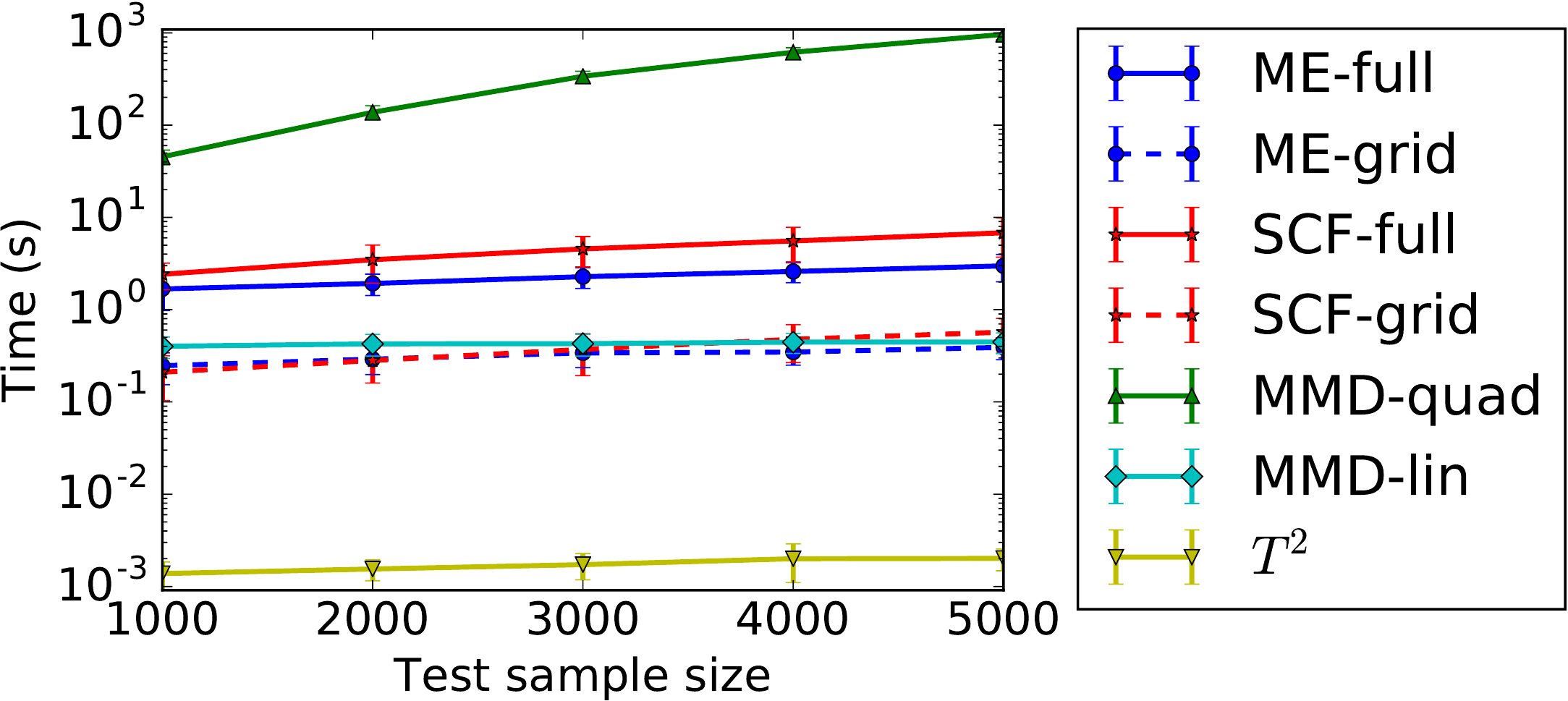} 
}

\caption{Plots of runtimes in the ``Test power vs. sample $n$'' experiment.
\label{fig:pow_vs_n_time}\vspace{-5mm}}
\end{figure}

\begin{figure}[th]
\centering
\subfloat[SG  \label{fig:ex2_sg_time}]{
\includegraphics[width=0.25\textwidth]{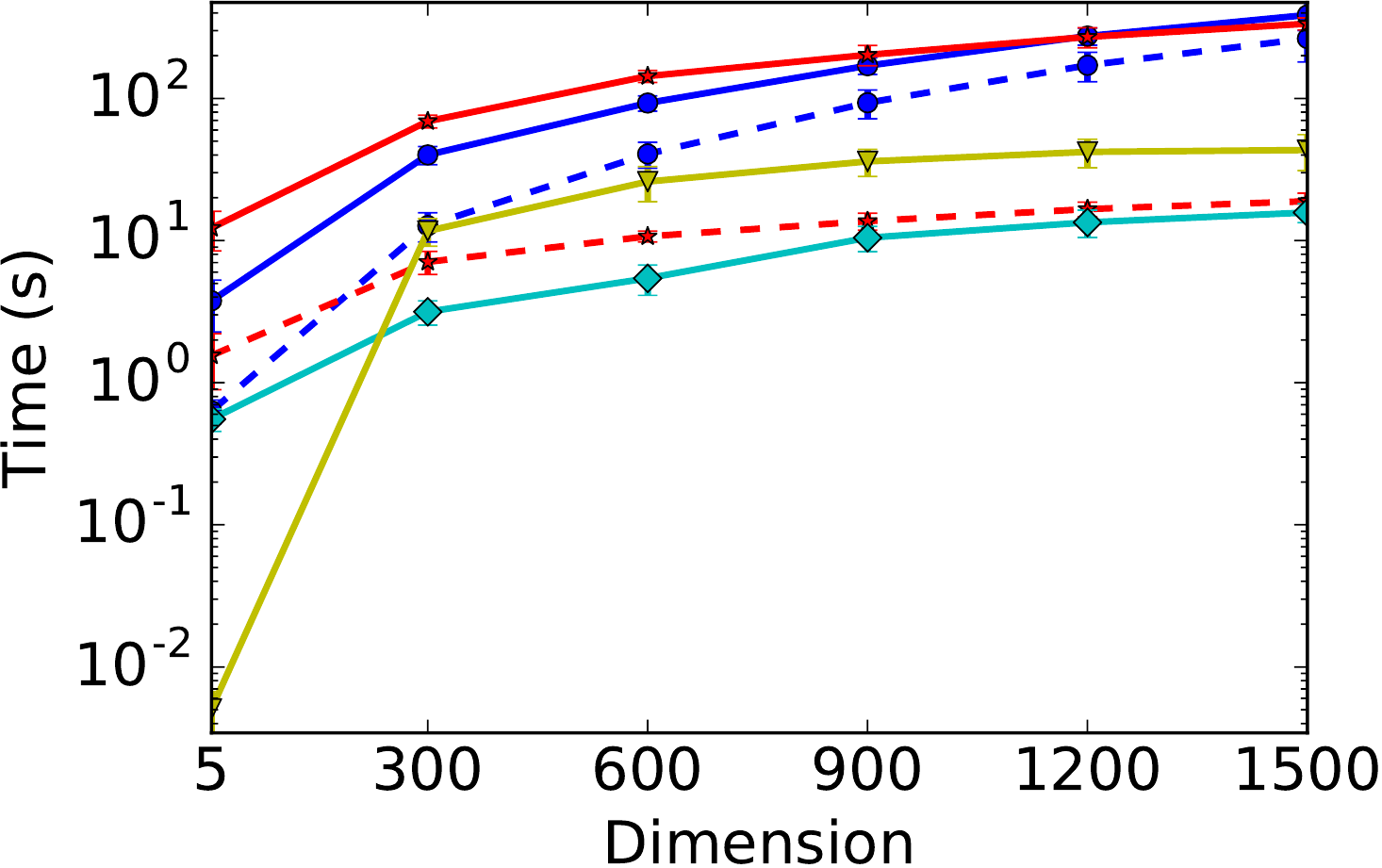}
}
\subfloat[GMD \label{fig:ex2_gmd_time}]{ \includegraphics[width=0.24\textwidth]{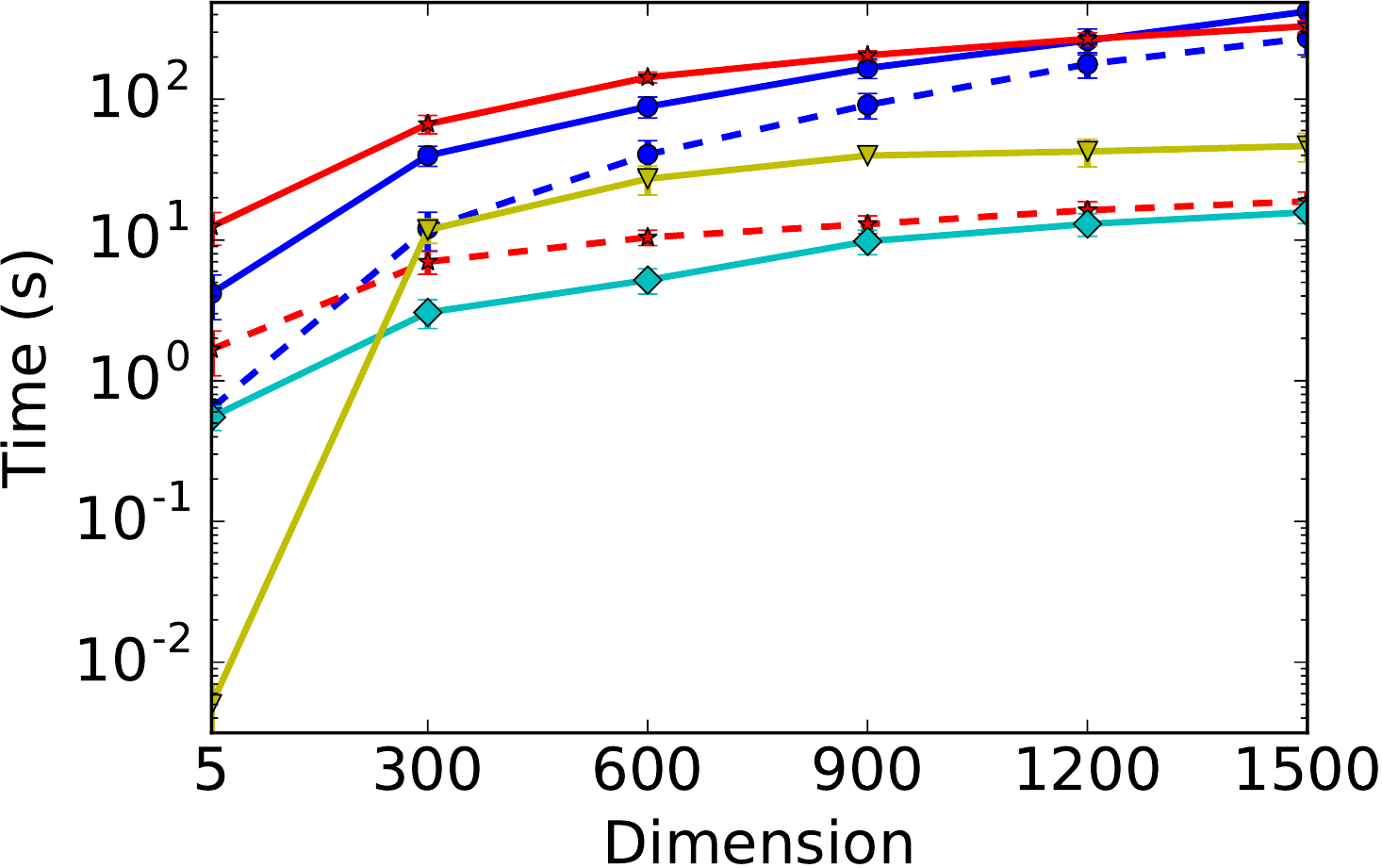} 
}
\subfloat[GVD \label{fig:ex2_gvd_time}]{ \includegraphics[width=0.36\textwidth]{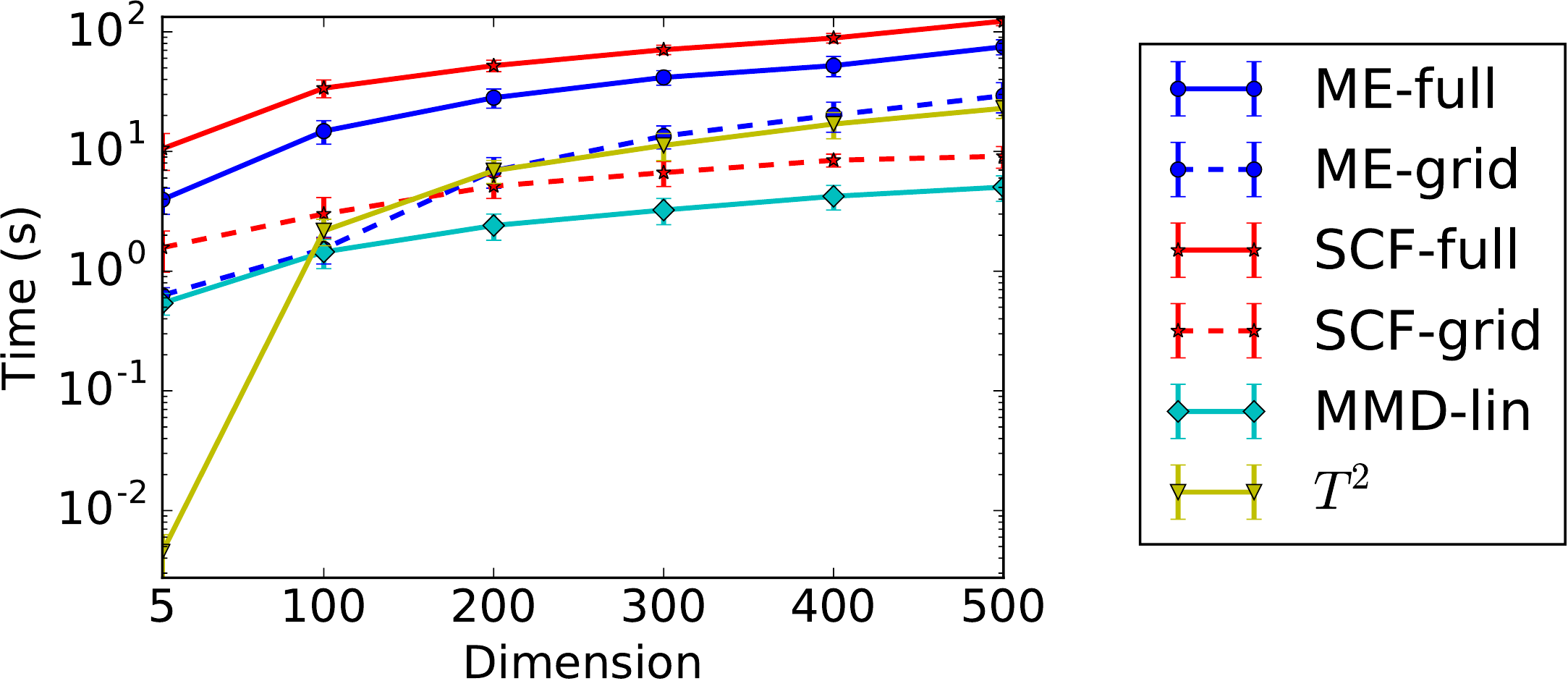} 
} 

\caption{Plots of runtimes in the ``Test power vs. dimension $d$'' experiment.
The test sample size is $10000$.\label{fig:power_vs_d_time} \vspace{-5mm}}
\end{figure}
\begin{table}[th]
\caption{Runtimes (in seconds) in the problem of distinguishing NIPS papers
from two categories.\label{tab:nips_time}}

\centering
\begin{tabular}{lc|llllll}
\toprule
Problem & $n^{te}$ & ME-full & ME-grid & SCF-full & SCF-grid & MMD-quad & MMD-lin\\
\midrule
Bayes-Bayes & 215 & 126.7 & 116 & 34.67 & 1.855 & 13.66 & .6112  \\ 
Bayes-Deep & 216 & 118.3 & 111.7 & 36.41 & 1.933 & 13.59 & .5105  \\ 
Bayes-Learn & 138 & 94.59 & 89.16 & 23.69 & 1.036 & 2.152 & .36  \\ 
Bayes-Neuro & 394 & 142.5 & 130.3 & 69.19 & 3.533 & 32.71 & .8643  \\ 
Learn-Deep & 149 & 105 & 99.59 & 24.99 & 1.253 & 2.417 & .4744  \\ 
Learn-Neuro & 146 & 101.2 & 93.53 & 25.29 & 1.178 & 2.351 & .3658  \\ 
\bottomrule
\end{tabular}
\end{table}
\begin{table}[H]
\caption{Runtimes (in seconds) in the problem of distinguishing positive (+)
and negative (-) facial expressions. \label{tab:face_time}}

\centering
\begin{tabular}{lc|cccccc}
\toprule
Problem & $n^{te}$ & ME-full & ME-grid & SCF-full & SCF-grid & MMD-quad & MMD-lin\\
\midrule
$\pm$ vs. $\pm$ & 201 & 87.7 & 83.4 & 10.5 &  1.45 &  9.93 & 0.464 \\ 
$+$ vs. $-$ & 201 & 85.0 & 80.6 & 11.7 & 1.42 & 10.4 & 0.482 \\ 
\bottomrule
\end{tabular}
\end{table}
In the cases where $n$ is large (Fig. \ref{fig:pow_vs_n_time}),
MMD-quad has the largest runtime due to its quadratic dependency on
the sample size. In the extreme case where the test sample size is
$10000$ (Fig. \ref{fig:power_vs_d_time}), it is computationally
infeasible to run MMD-quad. We observe that the proposed ME-full and
SCF-full have a slight overhead from the parameter optimization. However,
since the optimization procedure is also linear in $n$, we are able
to conduct an accurate test in less than 10 minutes even when the
test sample size is 10000 and $d=1500$ (see Fig. \ref{fig:ex2_sg_time},
\ref{fig:ex2_gmd_time}). We note that the actual tests (after optimization)
for all ME and SCF variants take less than one second in all cases.
In the ME-full, we initialize the test locations with realizations
from two multivariate normal distributions fitted to samples from
$P$ and $Q$. When $d$ is large, this heuristic can be expensive.
An alternative initialization scheme for $\mathcal{V}$ is to randomly
select $J$ points from the two samples.

\newpage

\section{Proof of theorem \ref{thm:lambda_conv_me}\label{sec:proof_lambda_conv_me}}

Recall Theorem\,\ref{thm:lambda_conv_me}: {\consistencyme*}

A proof is given as follows.

\subsection{Notations}

Let $\left<\b A,\b B\right>_F:=\mathrm{tr}\left(\b A\T \b B\right)$ 
be the Frobenius inner product, and $\left\|\b A\right\|_F:=\sqrt{\left<\b A,\b A\right>_F}$. $\b A\succeq \b 0$ means that $\b A\in \R^{d\times d}$ is symmetric, positive semidefinite.
For $\b a\in\R^d$, $\|\b a\|_2=\left<\b a, \b a\right>_2=\b a\T \b a$. $[\b a_1;\ldots;\b a_N]\in \R^{d_1+\ldots +d_N}$ is the concatenation of the $\b a_n\in \R^{d_n}$ vectors. $\R^+$ is the set of positive reals. $f\circ g$ is the composition of function $f$ and $g$. Let $\M$
denote a general metric space below. In measurability requirements
metric spaces are meant to be endowed with their Borel $\sigma$-algebras.

Let $\C$ be a collection of subsets
of $\M$ ($C\subseteq 2^{\M}$). $\C$
is said to shatter an $\left\{ p_1,p_2,\ldots,p_i\right\} \subseteq \M$
set, if for any $S\subseteq\left\{p_1,p_2,\ldots,p_i\right\}$
there exist $C\in\C$ such that $S = C\cap\left\{ p_1,p_2,\ldots,p_i\right\}$; 
in other words, arbitrary subset of $\left\{p_1,p_2,\ldots,p_i\right\}$ can be cut out by an element of $\C$. The VC index of $\C$
is the smallest $i$ for which no set of size i is shattered: 
\begin{align*}
    VC\left(\C\right) & =\inf\left\{ i:\max_{p_1,\ldots,p_i}\left|\left\{ C\cap\left\{p_1,\ldots,p_i\right\} :C\in\C\right\} \right|<2^i\right\} .
\end{align*}
A collection $\C$ of measurable sets is called VC-class if its index $VC\left(\C\right)$ is finite. The subgraph
of a real-valued function $f:\M\rightarrow\R$ is $sub(f)=\left\{(m,u):u<f(m)\right\} \subseteq\M\times\R$.
A collection of $\F$ measurable functions is called VC-subgraph
class, or shortly VC if the collection of all subgraphs of $\F$,
$\left\{sub(f)\right\}_{f\in\F}$ is a VC-class of sets;
its index is defined as $VC\left(\F\right):=VC\left(\left\{ sub(f)\right\}_{f\in\F}\right)$.

Let $L^{0}(\M)$ be the set of $\M\rightarrow\R$ measurable functions. 
Given an i.i.d.\ (independent identically distributed) sample from $\P$ ($w_i\stackrel{i.i.d.}{\sim}\P$), let
$w_{1:n}=\left(w_1,\ldots,w_n\right)$ and let $\P_n=\frac{1}{n}\sum_{i=1}^n\delta_{w_i}$
denote the empirical measure. $L^q\left(\M,\P_n\right)=\left\{f\in L^0\left(\M\right):\|f\|_{L^q(\M,\P_n)}:=\left[\int_{\M}|f(w)|^q\d \P_n(w)\right]^{\frac{1}{q}}=\left[\frac{1}{n}\sum_{i=1}^n|f(w_i)|^q\right]^{\frac{1}{q}}<\infty\right\}$ ($1\le q<\infty$), 
$\left\|f\right\|_{L^{\infty}(\M)}:=\sup_{m\in \M}|f(m)|$. 
Define $\P f:= \int_{\M}f(w)\d \P(w)$, where $\P$ is a probability distribution on $\M$.
Let $\left\|\P_n-\P\right\|_{\F}:= \sup_{f\in\F}|\P_n f- \P f|$.

The diameter of a class $\F\subseteq L^2\left(\M,\P_n\right)$ is $\diam(\F,L^2(\M,\P_n)):=\sup_{f,f'\in\F}\|f-f'\|_{L^2(\M,\P_n)}$,
its $r$-covering number ($r>0$) is the size of the smallest $r$-net 
\begin{align*}
    N\left(r,\F,L^2\left(\M,\P_n\right)\right) & =\inf\left\{ t\ge 1:\exists f_1,\ldots,f_t \in\F\text{ such that }\F\subseteq \cup_{i=1}^t B(r,f_i)\right\},
\end{align*}
where $B(r,f)=\left\{ g\in L^2\left(\M,\P_n\right)\mid \|f-g\|_{L^2(\M,\P_n)}\le r\right\}$
is the ball with center $f$ and radius $r$. $\times_{i=1}^N\Q_i$ is the $N$-fold
product measure. For sets $Q_i$, $\times_{i=1}^n Q_i$ is their
Cartesian product. For a function class $\F\subseteq L^0\left(\M\right)$
and $w_{1:n}\in\M^n$, $R(\F,w_{1:n}):=\E_{\b r}\left[\sup_{f\in\F}\left|\frac{1}{n}\sum_{i=1}^n r_i f(w_i)\right|\right]$
is the empirical Rademacher average, where $\b r:=r_{1:n}$ and $r_i$-s
are i.i.d.\ samples from a Rademacher random variable [$\P(r_i=1)=\P(r_i=-1)=\frac{1}{2}$]. Let $\left(\Theta,\rho\right)$
be a metric space; a collection of $\F=\left\{f_{\theta} \mid \theta\in\Theta\right\}\subseteq L^{0}(\M)$
functions is called a separable Carath{\'e}odory family if $\Theta$ is
separable and $\theta\mapsto f_{\theta}(m)$ is continuous for all
$m\in\M$. $span(\cdot)$ denotes the linear hull of its arguments.
$\Gamma(t)=\int_0^{\infty} u^{t-1} e^{-u}\d u$ denotes the Gamma function.

\subsection{Bound in terms of $\mathbf{S}_{n}$ and $\overline{\mathbf{z}}_{n}$}

For brevity, we will interchangeably use $\mathbf{S}_{n}$ for $\mathbf{S}_{n}(\mathcal{V})$
and \textbf{$\overline{\mathbf{z}}_{n}$} for $\overline{\mathbf{z}}_{n}(\mathcal{V})$.
 $\mathbf{S}_{n}(\mathcal{V})$ and $\overline{\mathbf{z}}_{n}(\mathcal{V})$
will be used mainly when the dependency of $\mathcal{V}$ needs to
be emphasized. We start with $\sup_{\mathcal{V},k}\left|\overline{\mathbf{z}}_{n}^{\top}(\mathbf{S}_{n}+\gamma_{n}I)^{-1}\overline{\mathbf{z}}_{n}-\boldsymbol{\mu}^{\top}\boldsymbol{\Sigma}^{-1}\boldsymbol{\mu}\right|$and
upper bound the argument of $\sup_{\mathcal{V},k}$ as
\begin{align*}
 & \left|\overline{\mathbf{z}}_{n}^{\top}(\mathbf{S}_{n}+\gamma_{n}I)^{-1}\overline{\mathbf{z}}_{n}-\boldsymbol{\mu}^{\top}\boldsymbol{\Sigma}^{-1}\boldsymbol{\mu}\right|\\
 & =\left|\overline{\mathbf{z}}_{n}^{\top}(\mathbf{S}_{n}+\gamma_{n}I)^{-1}\overline{\mathbf{z}}_{n}-\boldsymbol{\mu}^{\top}\left(\boldsymbol{\Sigma}+\gamma_{n}I\right)^{-1}\boldsymbol{\mu}+\boldsymbol{\mu}^{\top}\left(\boldsymbol{\Sigma}+\gamma_{n}I\right)^{-1}\boldsymbol{\mu}-\boldsymbol{\mu}^{\top}\boldsymbol{\Sigma}^{-1}\boldsymbol{\mu}\right|\\
 & \le\left|\overline{\mathbf{z}}_{n}^{\top}(\mathbf{S}_{n}+\gamma_{n}I)^{-1}\overline{\mathbf{z}}_{n}-\boldsymbol{\mu}^{\top}\left(\boldsymbol{\Sigma}+\gamma_{n}I\right)^{-1}\boldsymbol{\mu}\right|+\left|\boldsymbol{\mu}^{\top}\left(\boldsymbol{\Sigma}+\gamma_{n}I\right)^{-1}\boldsymbol{\mu}-\boldsymbol{\mu}^{\top}\boldsymbol{\Sigma}^{-1}\boldsymbol{\mu}\right|\\
 & :=(\square_{1})+(\square_{2}).
\end{align*}
For $(\square_{1})$, we have
\begin{align*}
 & \left|\overline{\mathbf{z}}_{n}^{\top}(\mathbf{S}_{n}+\gamma_{n}I)^{-1}\overline{\mathbf{z}}_{n}-\boldsymbol{\mu}^{\top}\left(\boldsymbol{\Sigma}+\gamma_{n}I\right)^{-1}\boldsymbol{\mu}\right|\\
 & =\left|\left\langle \overline{\mathbf{z}}_{n}\overline{\mathbf{z}}_{n}^{\top},(\mathbf{S}_{n}+\gamma_{n}I)^{-1}\right\rangle _{F}-\left\langle \boldsymbol{\mu}\boldsymbol{\mu}^{\top},\left(\boldsymbol{\Sigma}+\gamma_{n}I\right)^{-1}\right\rangle _{F}\right|\\
 & =\left|\left\langle \overline{\mathbf{z}}_{n}\overline{\mathbf{z}}_{n}^{\top},(\mathbf{S}_{n}+\gamma_{n}I)^{-1}\right\rangle _{F}-\left\langle \overline{\mathbf{z}}_{n}\overline{\mathbf{z}}_{n}^{\top},\left(\boldsymbol{\Sigma}+\gamma_{n}I\right)^{-1}\right\rangle _{F}+\left\langle \overline{\mathbf{z}}_{n}\overline{\mathbf{z}}_{n}^{\top},\left(\boldsymbol{\Sigma}+\gamma_{n}I\right)^{-1}\right\rangle _{F}-\left\langle \boldsymbol{\mu}\boldsymbol{\mu}^{\top},\left(\boldsymbol{\Sigma}+\gamma_{n}I\right)^{-1}\right\rangle _{F}\right|\\
 & \le\left|\left\langle \overline{\mathbf{z}}_{n}\overline{\mathbf{z}}_{n}^{\top},(\mathbf{S}_{n}+\gamma_{n}I)^{-1}-\left(\boldsymbol{\Sigma}+\gamma_{n}I\right)^{-1}\right\rangle _{F}\right|+\left|\left\langle \overline{\mathbf{z}}_{n}\overline{\mathbf{z}}_{n}^{\top}-\boldsymbol{\mu}\boldsymbol{\mu}^{\top},\left(\boldsymbol{\Sigma}+\gamma_{n}I\right)^{-1}\right\rangle _{F}\right|\\
 & =\|\overline{\mathbf{z}}_{n}\overline{\mathbf{z}}_{n}^{\top}\|_{F}\|(\mathbf{S}_{n}+\gamma_{n}I)^{-1}-\left(\boldsymbol{\Sigma}+\gamma_{n}I\right)^{-1}\|_{F}+\|\overline{\mathbf{z}}_{n}\overline{\mathbf{z}}_{n}^{\top}-\boldsymbol{\mu}\boldsymbol{\mu}^{\top}\|_{F}\|\left(\boldsymbol{\Sigma}+\gamma_{n}I\right)^{-1}\|_{F}\\
 & \stackrel{(a)}{\le}\|\overline{\mathbf{z}}_{n}\overline{\mathbf{z}}_{n}^{\top}\|_{F}\|(\mathbf{S}_{n}+\gamma_{n}I)^{-1}[\left(\boldsymbol{\Sigma}+\gamma_{n}I\right)-(\mathbf{S}_{n}+\gamma_{n}I)]\left(\boldsymbol{\Sigma}+\gamma_{n}I\right)^{-1}\|_{F}+\|\overline{\mathbf{z}}_{n}\overline{\mathbf{z}}_{n}^{\top}-\overline{\mathbf{z}}_{n}\boldsymbol{\mu}^{\top}+\overline{\mathbf{z}}_{n}\boldsymbol{\mu}^{\top}-\boldsymbol{\mu}\boldsymbol{\mu}^{\top}\|_{F}\|\boldsymbol{\Sigma}^{-1}\|_{F}\\
 & \stackrel{(a)}{\le}\|\overline{\mathbf{z}}_{n}\overline{\mathbf{z}}_{n}^{\top}\|_{F}\|(\mathbf{S}_{n}+\gamma_{n}I)^{-1}\|_{F}\|\boldsymbol{\Sigma}-\mathbf{S}_{n}\|_{F}\|\boldsymbol{\Sigma}^{-1}\|_{F}+\|\overline{\mathbf{z}}_{n}(\overline{\mathbf{z}}_{n}-\boldsymbol{\mu})^{\top}\|_{F}\|\boldsymbol{\Sigma}^{-1}\|_{F}+\|(\overline{\mathbf{z}}_{n}-\boldsymbol{\mu})\boldsymbol{\mu}^{\top}\|_{F}\|\boldsymbol{\Sigma}^{-1}\|_{F}\\
 & \stackrel{(b)}{\le}\frac{\sqrt{J}}{\gamma_{n}}\|\overline{\mathbf{z}}_{n}\|_{2}^{2}\|\boldsymbol{\Sigma}-\mathbf{S}_{n}\|_{F}\|\boldsymbol{\Sigma}^{-1}\|_{F}+\|\overline{\mathbf{z}}_{n}\|_{2}\|\overline{\mathbf{z}}_{n}-\boldsymbol{\mu}\|_{2}\|\boldsymbol{\Sigma}^{-1}\|_{F}+\|\boldsymbol{\mu}\|_{2}\|\overline{\mathbf{z}}_{n}-\boldsymbol{\mu}\|_{2}\|\boldsymbol{\Sigma}^{-1}\|_{F},
\end{align*}
where at (a) we use $\|\left(\boldsymbol{\Sigma}+\gamma_{n}I\right)^{-1}\|_{F}\le\|\boldsymbol{\Sigma}^{-1}\|_{F}$
and at (b) we use $\|(\mathbf{S}_{n}+\gamma_{n}I)^{-1}\|_{F}\le\sqrt{J}\|(\mathbf{S}_{n}+\gamma_{n}I)^{-1}\|_{2}\le\sqrt{J}/\gamma_{n}$.

For $(\square_{2})$, we have
\begin{align*}
\left|\boldsymbol{\mu}^{\top}\left(\boldsymbol{\Sigma}+\gamma_{n}I\right)^{-1}\boldsymbol{\mu}-\boldsymbol{\mu}^{\top}\boldsymbol{\Sigma}^{-1}\boldsymbol{\mu}\right| & =\left|\left\langle \boldsymbol{\mu}\boldsymbol{\mu}^{\top},\left(\boldsymbol{\Sigma}+\gamma_{n}I\right)^{-1}-\boldsymbol{\Sigma}^{-1}\right\rangle _{F}\right|\\
 & \le\|\boldsymbol{\mu}\boldsymbol{\mu}^{\top}\|_{F}\|\left(\boldsymbol{\Sigma}+\gamma_{n}I\right)^{-1}-\boldsymbol{\Sigma}^{-1}\|_{F}\\
 & =\|\boldsymbol{\mu}\|_{2}^{2}\|\left(\boldsymbol{\Sigma}+\gamma_{n}I\right)^{-1}\left[\boldsymbol{\Sigma}-(\boldsymbol{\Sigma}+\gamma_{n}I)\right]\boldsymbol{\Sigma}^{-1}\|_{F}\\
 & =\gamma_{n}\|\boldsymbol{\mu}\|_{2}^{2}\|\|\left(\boldsymbol{\Sigma}+\gamma_{n}I\right)^{-1}\boldsymbol{\Sigma}^{-1}\|_{F}\\
 & \le\gamma_{n}\|\boldsymbol{\mu}\|_{2}^{2}\|\|\left(\boldsymbol{\Sigma}+\gamma_{n}I\right)^{-1}\|_{F}\|\boldsymbol{\Sigma}^{-1}\|_{F}\\
 & \stackrel{(a)}{\le}\gamma_{n}\|\boldsymbol{\mu}\|_{2}^{2}\|\|\boldsymbol{\Sigma}^{-1}\|_{F}^{2}.
\end{align*}
 Combining the upper bounds for $(\square_{1})$ and $(\square_{2})$,
we arrive at 
\begin{align}
 & \left|\overline{\mathbf{z}}_{n}^{\top}(\mathbf{S}_{n}+\gamma_{n}I)^{-1}\overline{\mathbf{z}}_{n}-\boldsymbol{\mu}^{\top}\boldsymbol{\Sigma}^{-1}\boldsymbol{\mu}\right|\nonumber \\
 & \le\frac{\sqrt{J}}{\gamma_{n}}\|\overline{\mathbf{z}}_{n}\|_{2}^{2}\|\boldsymbol{\Sigma}-\mathbf{S}_{n}\|_{F}\|\boldsymbol{\Sigma}^{-1}\|_{F}+(\|\overline{\mathbf{z}}_{n}\|_{2}+\|\boldsymbol{\mu}\|_{2})\|\overline{\mathbf{z}}_{n}-\boldsymbol{\mu}\|_{2}\|\boldsymbol{\Sigma}^{-1}\|_{F}+\gamma_{n}\|\boldsymbol{\mu}\|_{2}^{2}\|\|\boldsymbol{\Sigma}^{-1}\|_{F}^{2}\nonumber \\
 & \le4B^{2}J\tilde{c}\frac{\sqrt{J}}{\gamma_{n}}\|\boldsymbol{\Sigma}-\mathbf{S}_{n}\|_{F}+4B\sqrt{J}\tilde{c}\|\overline{\mathbf{z}}_{n}-\boldsymbol{\mu}\|_{2}+4B^{2}J\tilde{c}^{2}\gamma_{n}\nonumber \\
 & =\frac{\overline{c}_{1}}{\gamma_{n}}\|\boldsymbol{\Sigma}-\mathbf{S}_{n}\|_{F}+\overline{c}_{2}\|\overline{\mathbf{z}}_{n}-\boldsymbol{\mu}\|_{2}+\overline{c}_{3}\gamma_{n}\label{eq:lamb_minus_lamb_h2}
\end{align}
with $\overline{c}_{1}:=4B^{2}J\sqrt{J}\tilde{c},\overline{c}_{2}:=4B\sqrt{J}\tilde{c},$
$\overline{c}_{3}:=4B^{2}J\tilde{c}^{2}$, and $\tilde{c}:=\sup_{\mathcal{V},k}\|\boldsymbol{\Sigma}^{-1}\|_{F}<\infty$,
where we applied the triangle inequality, the CBS (Cauchy-Bunyakovskii-Schwarz)
inequality, and $\|\mathbf{a}\mathbf{b}^{\top}\|_{F}=\|\mathbf{a}\|_{2}\|\mathbf{b}\|_{2}$.
The boundedness of kernel $k$ with the Jensen inequality implies
that 

 	  \begin{align}
	      \|\bar{\b z}_n\|_2^{2}&= \|\frac{1}{n}\sum_{i=1}^n \b z_i\|_2^2\leq\frac{1}{n}\sum_{i=1}^n\|\b z_i\|_2^2 
	= \frac{1}{n}\sum_{i=1}^n\|(k(\b x_i,\b v_j) - k(\b y_i,\b v_j))_{j=1}^J\|_2^2 \\
	& = \frac{1}{n}\sum_{i=1}^n\sum_{j=1}^J \left[k(\b x_i,\b v_j) - k(\b y_i,\b v_j)\right]^2\nonumber\\
		  &\le \frac{2}{n}\sum_{i=1}^n\sum_{j=1}^J k^2(\b x_i,\b v_j)+k^2(\b y_i,\b v_j)\le 4B^2J, \label{eq:z_n-bar-norm} \\
		\left\|\bm \mu(\V)\right\|_2^2 &= \sum_{j=1}^J  \left(\E_{\b{xy}}\left[k(\b x,\b v_j)-k(\b y,\b v_j)\right]\right)^2 \le \sum_{j=1}^J  \E_{\b{xy}}\left[k(\b x,\b v_j)-k(\b y,\b v_j)\right]^2 \le 4B^2J. \label{eq:mu-bar-norm}
	  \end{align}
Taking $\sup$ in (\ref{eq:lamb_minus_lamb_h2}), we get 
\begin{align*}
\sup_{\mathcal{V},k}\left|\overline{\mathbf{z}}_{n}^{\top}(\mathbf{S}_{n}+\gamma_{n}I)^{-1}\overline{\mathbf{z}}_{n}-\boldsymbol{\mu}^{\top}\boldsymbol{\Sigma}^{-1}\boldsymbol{\mu}\right| & \le\frac{\overline{c}_{1}}{\gamma_{n}}\sup_{\mathcal{V},k}\|\boldsymbol{\Sigma}-\mathbf{S}_{n}\|_{F}+\overline{c}_{2}\sup_{\mathcal{V},k}\|\overline{\mathbf{z}}_{n}-\boldsymbol{\mu}\|_{2}+\overline{c}_{3}\gamma_{n}.
\end{align*}

      \subsection{Empirical process bound on $\bar{\b z}_n$  }
	Recall that $\bar{\b z}_n\left(\V\right)=\frac{1}{n}\sum_{i=1}^n\b z_i\left(\V\right)\in\R^J$,
	$\b z_i\left(\V\right)=(k(\b x_i,\b v_j)-k(\b y_i,\b v_j))_{j=1}^J\in\R^J$,
	$\bm \mu(\V)=\left(\E_{\b{xy}}\left[k(\b x,\b v_j)-k(\b y,\b v_j)\right]\right)_{j=1}^J$;
	thus 
	\begin{align*}
		    \sup_{\V}\sup_{k\in\K}\|\bar{\b z}_n(\V)-\bm \mu(\V)\|_{2} & = \sup_{\V}\sup_{k\in\K}\sup_{\b b \in B(1,\b 0)}\left<\b b, \bar{\b z}_n(\V)-\bm \mu(\V)\right>_2
	\end{align*}

	  using that $\left\| \b a \right\|_2 = \sup_{\b b \in B(1,\b 0)} \left<\b a,\b b\right>_2$. Let us bound the argument of the supremum:
	\begin{align}
	      \left<\b b, \bar{\b z}_n(\V)-\bm \mu(\V)\right>_2  & \le \sum_{j=1}^J |b_j| \left| \frac{1}{n} \sum_{i=1}^n \left[k(\b x_i, \b v_j) - k(\b y_i,\b v_j)\right] - \E_{\b{xy}} \left[k(\b x, \b v_j) - k(\b y,\b v_j) \right]\right|\nonumber\\
		&\le  \sum_{j=1}^J |b_j| \left( \left|\frac{1}{n}\sum_{i=1}^nk(\b x_i,\b v_j)-\E_{\b{x}}k(\b x,\b v_j)\right| + \left|\frac{1}{n}\sum_{i=1}^nk(\b y_i,\b v_j)-\E_{\b{y}}k(\b y,\b v_j)\right|\right)\nonumber\\
	      &\le \sqrt{J}\sup_{\b v \in \X}\sup_{k\in\K}\left|\frac{1}{n}\sum_{i=1}^nk(\b x_i,\b v)-\E_{\b{x}}k(\b x,\b v)\right| + \sqrt{J}\sup_{\b v\in\X}\sup_{k\in\K}\left|\frac{1}{n}\sum_{i=1}^nk(\b y_i,\b v)-\E_{\b{y}}k(\b y,\b v)\right|\nonumber\\
	      &= \sqrt{J}  \left\|P_n- P\right\|_{\F_1} + \sqrt{J} \left\|Q_n- Q\right\|_{\F_1} \label{eq:mean-final}
	\end{align}
	by the triangle inequality and exploiting that $\left\|\b b\right\|_1\le \sqrt{J}\left\|\b  b\right\|_2 \le \sqrt{J}$ with $\b b \in B(1,\b 0)$.
    Thus, we have 
    \begin{align*}
            \sup_{\V}\sup_{k\in\K}\|\bar{\b z}_n(\V)-\bm \mu(\V)\|_{2} & \le
          \sqrt{J}  \left\|P_n- P\right\|_{\F_1} + \sqrt{J} \left\|Q_n-
          Q\right\|_{\F_1}.
    \end{align*}

    \subsection{Empirical process bound on $\b S_n$}
	Noting that 
	  \begin{align*}
	   \bm \Sigma(\V) &= \E_{\b{xy}}\left[ \b z(\V) \b z\T(\V)\right] - \bm \mu (\V) \bm \mu\T (\V), &
	    \b S_n(\V) &= \frac{1}{n}\sum_{a=1}^n \b z_a(\V) \b{z}_a\T(\V) - \frac{1}{n(n-1)}\sum_{a=1}^n\sum_{b\ne a}\b{z}_a \b{z}_b^T,\\
	    \E_{\b{xy}}\left[ \b z(\V) \b z\T(\V)\right] &= \E_{\b{xy}} \left[ \frac{1}{n}\sum_{a=1}^n \b z_a(\V) \b{z}_a\T(\V) \right], & 
	    \bm \mu (\V) \bm \mu\T (\V) &= \E_{\b{xy}} \left[ \frac{1}{n(n-1)}\sum_{a=1}^n\sum_{b\ne a}\b{z}_a(\V) \b{z}^T_b(\V) \right],
	  \end{align*}
	  we bound our target quantity as 
	\begin{align}
	  \|\b S_n(\V) - \bm \Sigma(\V)\|_F
	    &\le \left\| \frac{1}{n}\sum_{a=1}^n \b z_a(\V) \b{z}_a\T(\V) - \E_{\b{xy}}\left[ \b z(\V) \b z\T(\V)\right]\right\|_F + \left\| \frac{1}{n(n-1)}\sum_{a=1}^n\sum_{b\ne a}\b{z}_a(\V) \b{z}_b^T(\V) -  \bm \mu (\V) \bm \mu\T (\V)  \right\|_F\nonumber\\
	    &=:(*_1)+ (*_2). \label{eq:cov-decomp}
 	\end{align}
	\vspace*{-0.6cm}
	\begin{align}
	   (*_2)&= \left\| \frac{1}{n}\sum_{a=1}^n\b{z}_a(\V) \left[\frac{1}{n-1}\sum_{b\ne a} \b{z}_b\T(\V)\right] -  \bm \mu (\V) \bm \mu\T (\V)  \right\|_F \nonumber \\
	        &\le  \left\| \frac{1}{n}\sum_{a=1}^n\b{z}_a(\V) \left(\frac{1}{n-1}\sum_{b\ne a} \b{z}_b\T(\V) - \bm \mu\T(\V)\right)\right\|_F + \left\|\left(\frac{1}{n}\sum_{a=1}^n\b{z}_a(\V)  -\bm \mu (\V)\right) \bm \mu\T (\V)  \right\|_F\nonumber\\
        &\le \left\| \left(\frac{1}{n}\sum_{a=1}^n\b{z}_a(\V)\right)
        \left(\frac{1}{n-1}\sum_{b=1}^n \b{z}_b(\V) - \bm \mu(\V)\right)\T
        \right\|_F +  \left\|\left(\frac{1}{n}\sum_{a=1}^n\b{z}_a(\V)\right)
        \frac{\b z_a\T(\V)}{n-1}\right\|_F \nonumber\\ 
        & \quad + \left\|\left(\frac{1}{n}\sum_{a=1}^n\b{z}_a(\V)  -\bm \mu (\V)\right)
        \bm \mu\T (\V)  \right\|_F\nonumber\\
		&= \left\| \bar{\b z}_n(\V) \right\|_2 \left\|\frac{1}{n-1}\sum_{b=1}^n \b{z}_b(\V) - \bm \mu(\V) \right\|_2  + \frac{1}{n-1}\left\|\bar{\b z}_n(\V)\right\|_2 \left\|\b z_a(\V)\right\|_2 + \left\|\bar{\b z}_n(\V)-\bm \mu (\V)\right\|_2 \left\|\bm \mu (\V)  \right\|_2\nonumber\\
		&\le 2B\sqrt{J} \left( \frac{n}{n-1} \left\|\bar{\b z}_n  - \bm \mu(\V)\right\|_2 + \frac{2B\sqrt{J}}{n-1}  \right) + \frac{1}{n-1}4B^2 J + 2B\sqrt{J} \left\|\bar{\b z}_n(\V)-\bm \mu (\V)\right\|_2 \nonumber\\
        &= \frac{8 B^2J}{n-1} + 2B\sqrt{J} \frac{2n-1}{n-1} \left\|\bar{\b z}_n  - \bm \mu(\V)\right\|_2 \label{eq:me_cov_star2}
	\end{align}
    using the triangle inequality, the sub-additivity of $\sup$, $\left\|\b{a} \b{b}^T\right\|_F = \left\| \b a \right\|_2 \left\| \b b \right\|_2$, $\left\|\bar{\b z}_n(\V)\right\|_2 \le 2B\sqrt{J}$, $\left\|\b z_a(\V)\right\|_2 \le 2B\sqrt{J}$ [see Eq.~\eqref{eq:z_n-bar-norm}] and 
    \begin{align*}
	\left\|\frac{1}{n-1}\sum_{b=1}^n \b{z}_b(\V) - \bm \mu(\V) \right\|_2 & = \left\|\frac{n}{n-1}\bar{\b z}_n  - \frac{n}{n-1}\bm \mu(\V)+\frac{1}{n-1}\bm \mu(\V)\right\|_2 \le \frac{n}{n-1} \left\|\bar{\b z}_n  - \bm \mu(\V)\right\|_2 + \frac{1}{n-1}\left\|\bm \mu(\V)\right\|_2
    \end{align*}
    with Eq.~\eqref{eq:mu-bar-norm}. 
    Considering the first term in Eq.~\eqref{eq:cov-decomp}
    \begin{align*}
      & \lefteqn{\left\| \frac{1}{n}\sum_{a=1}^n \b z_a(\V) \b{z}_a\T(\V) -
      \E_{\b{xy}}\left[ \b z(\V) \b z\T(\V)\right]\right\|_F = \sup_{\b B \in
      B(1,\b 0)} \left< \b B, \frac{1}{n}\sum_{a=1}^n \b z_a(\V) \b{z}_a\T(\V)
      - \E_{\b{xy}}\left[ \b z(\V) \b z\T(\V)\right] \right>_F }\\
      \le & \sup_{\b B \in B(1,\b 0)} \sum_{i,j=1}^J |B_{ij}|
      \left|\frac{1}{n}\sum_{a=1}^n [k(\b x_a,\b v_i)-k(\b y_a,\b v_i)][k(\b
      x_a,\b v_j)-k(\b y_a,\b v_j)] - \E_{\b{xy}}\left(\left[k(\b x,\b
      v_i)-k(\b y,\b v_i)\right] \left[k(\b x,\b v_j)-k(\b y,\b
      v_j)\right]\right)\right|\\
       \le  &\sup_{\b B \in B(1,\b 0)} \sum_{i,j=1}^J |B_{ij}|
        \left(\left|\frac{1}{n}\sum_{a=1}^n k(\b x_a,\b v_i) k(\b x_a,\b v_j) -
        \E_{\b{x}}\left[k(\b x,\b v_i) k(\b x,\b v_j)\right]\right|  \right. \\
        & \left. + \left|\frac{1}{n}\sum_{a=1}^n k(\b x_a,\b v_i) k(\b y_a,\b v_j) -
        \E_{\b{xy}}\left[k(\b x,\b v_i) k(\b y,\b v_j)\right]\right|  \right.\\
        &+\left.\left|\frac{1}{n}\sum_{a=1}^n k(\b y_a,\b v_i) k(\b x_a,\b v_j)
        - \E_{\b{xy}}\left[k(\b y,\b v_i) k(\b x,\b v_j)\right]\right|
		 +\left|\frac{1}{n}\sum_{a=1}^n k(\b y_a,\b v_i) k(\b y_a,\b v_j) - \E_{\b{y}}\left[k(\b y,\b v_i) k(\b y,\b v_j)\right]\right| \right)\\
	     \le & 
         J \sup_{\b v,\b v'\in \X}\sup_{k\in\K}  \left|\frac{1}{n}\sum_{a=1}^n
         k(\b x_a,\b v) k(\b x_a,\b v') - \E_{\b{x}}\left[k(\b x,\b v) k(\b
         x,\b v')\right]\right| \\
        & + 2J \sup_{\b v,\b v'\in \X}\sup_{k\in\K} \left|\frac{1}{n}\sum_{a=1}^n
        k(\b x_a,\b v) k(\b y_a,\b v') - \E_{\b{xy}}\left[k(\b x,\b v) k(\b
        y,\b v')\right]\right|\\
        &+ J \sup_{\b v,\b v'\in \X}\sup_{k\in\K} \left|\frac{1}{n}\sum_{a=1}^n
        k(\b y_a,\b v) k(\b y_a,\b v') - \E_{\b{y}}\left[k(\b y,\b v) k(\b y,\b
        v')\right]\right|
    \end{align*}
    by exploiting that $\left\|\b A\right\|_F = \sup_{\b B\in B(1,\b 0)} \left<\b B,\b A\right>_F$, and $\sum_{i,j=1}^J|B_{ij}| \le J \left\|\b B\right\|_F \le J$ with  $\b B \in B(1,\b 0)$.
    Using the bounds obtained for the two terms of Eq.~\eqref{eq:cov-decomp}, we get
    \begin{align}
      \lefteqn{\sup_{\V}\sup_{k\in\K} \|\b S_n(\V) - \bm \Sigma(\V)\|_F\le }\nonumber\\
	  & \le  \frac{8 B^2J}{n-1} + 2B\sqrt{J} \frac{2n-1}{n-1} \sup_{\V}\sup_{k\in\K} \left\|\bar{\b z}_n  - \bm \mu(\V)\right\|_2+
	    J \left( \left\|P_n-P\right\|_{\F_2}+ 2\left\|(P\times Q)_n - (P\times Q)\right\|_{\F_3}  + \left\| Q_n - Q\right\|_{\F_2}\right). \label{eq:cov-final}
    \end{align}

    \subsection{Bounding by concentration and the VC property}
 By combining Eqs.~\eqref{eq:lamb_minus_lamb_h2}, \eqref{eq:mean-final} and \eqref{eq:cov-final}
      \begin{eqnarray}
      \lefteqn{\sup_{\V}\sup_{k} \left|\bar{\b z}_n^{\top} (\b S_n +
      \gamma_n I)^{-1}\bar{\b z}_n- \bm \mu^{\top}\bm
      \Sigma^{-1}\bm \mu\right| \le} \nonumber\\
      &&\hspace*{-0.65cm}\le  \frac{\bar{c}_1}{\gamma_n} \bigg[ \frac{8 B^2J}{n-1} + 2B\sqrt{J}
      \frac{2n-1}{n-1} \sqrt{J}  \left(\left\|P_n- P\right\|_{\F_1} +  \left\|Q_n-
  Q\right\|_{\F_1}\right)  \nonumber \\ 
  &&+ J \left( \left\|P_n-P\right\|_{\F_2}+
  2\left\|(P\times Q)_n - (P\times Q)\right\|_{\F_3}  + \left\| Q_n -
  Q\right\|_{\F_2}\right) \bigg]\nonumber  \\
  && +\bar{c}_2 \sqrt{J}  \left(\left\|P_n- P\right\|_{\F_1} +  \left\|Q_n-
  Q\right\|_{\F_1}\right)  + \bar{c}_3 \gamma_n \nonumber  \\
      &&\hspace*{-0.65cm}=  \left(\left\|P_n- P\right\|_{\F_1} + \left\|Q_n-
      Q\right\|_{\F_1}\right) \left(\frac{2}{\gamma_n} \bar{c}_1 BJ \frac{2n-1}{n-1} +
      \bar{c}_2 \sqrt{J} \right)  + \bar{c}_3 \gamma_n \nonumber \\ 
      && +\frac{\bar{c}_1}{\gamma_n} J \left[\left\|P_n- P\right\|_{\F_2} + \left\|Q_n-
      Q\right\|_{\F_2} + 2\left\|(P\times Q)_n - (P\times Q)\right\|_{\F_3}
  \right]   +\frac{8}{\gamma_n} \frac{ \bar{c}_1B^2J}{n-1}. \label{eq:emp-process-bound}
      \end{eqnarray}
      Applying Lemma~\ref{lemma:emp-proc-general} with $\frac{\delta}{5}$, we get the statement with a union bound.  
      \QEDB


\vspace{5mm}

\begin{lem}[Concentration of the empirical process for uniformly bounded separable Carath{\'e}odory VC classes]\label{lemma:emp-proc-general}
Let $\F$  be 
\begin{enumerate}
    \item VC-subgraph class of $\M\rightarrow \R$ functions with VC index VC($\F$),
    \item a uniformly  bounded ($\left\|f\right\|_{L^{\infty}(\M)}\le K<\infty, \forall f\in \F$) separable Carath{\'e}odory family.
\end{enumerate}
Let $\Q$ be a probability measure, and let $\Q_n=\frac{1}{n}\sum_{i=1}^n
\delta_{x_i}$ be the corresponding empirical measure. Then for any $\delta \in
(0,1)$ with probability at least 
$1-\delta$
\begin{align*}
      \left\|\Q_n - \Q\right\|_{\F} &\le \frac{16\sqrt{2}K}{\sqrt{n}} \left[ 2\sqrt{\log\left[ C \times VC(\F) (16e)^{VC(\F)} \right]} + \frac{\sqrt{2\pi[VC(\F)-1]}}{2} \right] + K\sqrt{\frac{2\log\left(\frac{1}{\delta}\right)}{n}}
\end{align*}
where the universal constant $C$ is associated according to Lemma~\ref{lemma:VC-properties}(iv).
\end{lem}

\begin{proof}
    \label{proof:emp-proc-general}
Notice that $g(x_1,\ldots,x_n)=\left\|\Q_n - \Q\right\|_{\F}$ satisfies the bounded difference property with $b=\frac{2K}{n}$ [see Eq.~\eqref{eq:bounded-diff-property}]:
	  \begin{align*}
	      & |g(\b x_1,\ldots,\b x_n) - g\left(\b x_1,\ldots,\b x_j,\b x_j',\b x_{j+1},\ldots,\b x_n\right)| \\
	      &\le \left| \sup_{f\in \F}\Big|\Q f-\frac{1}{n}\sum_{i=1}^nf(\b x_i)\Big| - \sup_{f\in \F}\Big|\Q f-\frac{1}{n}\sum_{i=1}^nf(\b x_i)+\frac{1}{n}\left[f(\b x_j)-f(\b x_j')\right]\Big|\right|\\
	      &\le  \frac{1}{n} \sup_{f\in \F}|f(\b x_j)-f(\b x_j')| \le  \frac{1}{n} \left( \sup_{f\in \F}|f(\b x_j)| + \sup_{f\in \F}|f(\b x_j')| \right) \le \frac{2K}{n}.
	  \end{align*}
Hence, applying Lemma~\ref{lemma:McDiarmid}, and using symmetrization \cite{steinwart08support} (Prop.~7.10) for the uniformly bounded separable Carath{\'e}odory $\F$ class, for arbitrary $\delta\in (0,1)$ with probability at least $1-\delta$
\begin{align*}
  \left\|\Q_n - \Q\right\|_{\F} & \le  \E_{x_{1:n}} \left\|\Q_n -
  \Q\right\|_{\F} +  K\sqrt{\frac{2\log\left(\frac{1}{\delta}\right) }{n}} \\
  & \le 2 \E_{x_{1:n}}R(\F,x_{1:n}) +  K\sqrt{\frac{2\log\left(\frac{1}{\delta}\right) }{n}}. 
\end{align*}

By the Dudley entropy bound \cite{bousquet03new} [see Eq.~(4.4); $\diam(\F,L^2(\M,\Q_n))\le 2 \sup_{f\in \F}\|f\|_{L^2(\M,\Q_n)}\le 2 \sup_{f\in \F}\|f\|_{L^{\infty}(\M)} \le 2 K <\infty$], Lemma~\ref{lemma:VC-properties}(iv) [with $F\equiv K$, $q=2$  $\Me=\Q_n$] and the monotone decreasing property of the covering 
number, one arrives at 
\begin{align*}
      R(\F,x_{1:n}) &\le \frac{8\sqrt{2}}{\sqrt{n}} \int_0^{2K} \sqrt{\log N(r,\F,L^2(\M,\Q_n))} \d r \\
      & \le \frac{8\sqrt{2}}{\sqrt{n}} \left[\int_0^K \sqrt{\log N(r,\F,L^2(\M,\Q_n))} \d r + K \sqrt{\log N(K,\F,L^2(\M,\Q_n))}  \right]\\
	  &\le \frac{8\sqrt{2}K}{\sqrt{n}} \left[\int_0^1 \sqrt{\log N(rK,\F,L^2(\M,\Q_n))} \d r + \sqrt{\log N(K,\F,L^2(\M,\Q_n))}  \right]\\
	  &\le \frac{8\sqrt{2}K}{\sqrt{n}} \left[  \int_0^1 \sqrt{\log \left[a_1\left(\frac{1}{r}\right)^{a_2}\right]} \d r + \sqrt{\log(a_1)}\right]
	  = \frac{8\sqrt{2}K}{\sqrt{n}} \left[ 2\sqrt{\log(a_1)} + \int_0^1 \sqrt{a_2 \log\left(\frac{1}{r}\right)}\d r\right]\\
	  &=\frac{8\sqrt{2}K}{\sqrt{n}} \left[ 2\sqrt{\log(a_1)} + \frac{\sqrt{\pi a_2}}{2} \right],
\end{align*}
where $a_1:=C \times VC(\F) (16e)^{VC(\F)}$, $a_2 := 2[VC(\F)-1]$
and $\int_0^1\sqrt{\log\left(\frac{1}{r}\right)} \d r = \int_0^{\infty}
t^{\frac{1}{2}} e^{-t}\d t = \Gamma\left(\frac{3}{2}\right) =
\frac{\sqrt{\pi}}{2}$. 
\end{proof}

\begin{lem}[Properties of $\F_i$  from $\K$]~\label{lemma:F_i-properties}
  \begin{enumerate}
      \item \textbf{Uniform boundedness of $\F_i$-s [see
          Eqs.~\eqref{eq:F_1-2}-\eqref{eq:F_3}]:} If $\K$ is uniformly bounded,
          i.e., $\exists B<\infty$ such that $\sup_{k\in\K}\sup_{(\b x,\b
          y)\in\X^2}\left|k(\b x,\b y)\right|\le B$; then $\F_1$, $\F_2$ and
          $\F_3$ [Eqs.~\eqref{eq:F_1-2}-\eqref{eq:F_3}] are also uniformly
          bounded with $B$, $B^2$, $B^2$ constants, respectively. That is, $\sup_{k\in
          \K, \b v \in \X}|k(\b x, \b v)| \le B$, $\sup_{k\in \K, (\b v,\b v')
          \in \X^2}|k(\b x, \b v)k(\b x,\b v')| \le B^2$,
	    $\sup_{k\in \K, (\b v,\b v') \in \X^2}|k(\b x, \b v)k(\b y,\b v')| \le B^2$.
    \item \textbf{Separability of  $\F_i$:} since $\F_1$, $\F_2$ and $\F_3$ is
          parameterized by $\Theta=\K\times \X$, $\K\times \X^2$, $\K\times
          \X^2$, separability of $\K$  implies that of $\Theta$.
      \item \textbf{Measurability of $\F_i$:}  $\forall k\in\K$ is measurable, then the
          elements of $\F_i$ ($i=1,2,3$) are also measurable. \QEDB
  \end{enumerate} 
\end{lem}

\section{Example kernel families} \label{sec:example-kernel-families}
Below we give examples for $\K$ kernel classes for which the associated $\F_i$-s are VC-subgraph and uniformly bounded separable Carath{\'e}odory families. The VC property will be a direct consequence 
of the  VC indices of finite-dimensional function classes and preservation theorems (see Lemma~\ref{lemma:VC-properties}); for a nice example application see 
\cite{srebro06learning} (Section 5) who study the pseudo-dimension of $(\b x,\b y) \mapsto k(\b x,\b y)$ kernel classes, for different Gaussian families. We take these Gaussian classes (isotropic, full) and use the preservation trick
to bound the VC indices of the associated $\F_i$-s.

\begin{lem}[$\F_i$-s are VC-subgraph and uniformly bounded separable Carath{\'e}odory families for isotropic Gaussian kernel] \label{lemma:isotropic-Gaussian-properties}
    Let $\K=\left\{k_{\sigma}: (\b x,\b y)\in \X \times \X \subseteq \R^d \times \R^d \mapsto e^{-\frac{\left\|\b x-\b y\right\|_2^2}{2\sigma^2}}:\sigma>0 \right\}$. Then the $\F_1$, $\F_2$, $\F_3$ classes [see Eqs.~\eqref{eq:F_1-2}-\eqref{eq:F_3}] associated to $\K$ are
    \begin{itemize}
    \item VC-subgraphs with indices $VC(\F_1)\le d+4$, $VC(\F_2)\le d+4$,
        $VC(\F_3)\le 2d+4$, and
    \item uniformly bounded separable Carath{\'e}odory families, with
        $\left\|f\right\|_{L^\infty(\M)}\le 1$  for all $f\in
        \{\F_1, \F_2, \F_3\}$.\footnote{$\M=\X$ for $\F_1$ and $\F_2$, and $\M=\X^2$
            in case of $\F_3$. \label{footnote:M-def}}
    \end{itemize} 
\end{lem}

\begin{proof}
    \textbf{ VC subgraph property:} 
    \begin{itemize}
        \item $\F_1$: Consider the function class $\G= \left\{ \b x\mapsto
                \frac{\left\|\b x-\b v\right\|_2^2}{2\sigma^2} =
                \frac{1}{2\sigma^2}\left(\left\|\b{x}\right\|_2^2 - 2 \left<\b
                x, \b v\right>_2 + \left\|\b v\right\|_2^2\right):\sigma>0,
                \b{v}\in\X\right\}\subseteq L^0(\R^{d})$. 
            $\G\subseteq \tilde{\G}:=span\left(\b x\mapsto \left\|\b
            x\right\|_2^2, \{ \b x\mapsto x_i \}_{i=1}^d \hspace{0.15cm}, \b
            x\mapsto 1\right)$ vector space, $dim(\G)\le d+2$. Thus by
            Lemma~\ref{lemma:VC-properties}(i)-(ii), 
		    $\G$ is VC with $VC(\G)\le d+4$; applying Lemma~\ref{lemma:VC-properties}(iii) with $\phi(z)=e^{-z}$, $\F_1 = \phi \circ \G$ is also VC with index $VC(\F_1)\le d+4$. 
		\item $\F_2$: Since $\F_2 = \left\{\b x \mapsto k(\b x ,\b v)k(\b x,\b v')=e^{-\frac{\left\|\b x-\b v\right\|_2^2 + \left\|\b x-\b v'\right\|_2^2}{2\sigma^2}}: \sigma>0, \b v\in \X,\b v'\in\X\right\}$, and 
              $\left\{\b x\mapsto \frac{\left\|\b x-\b v\right\|_2^2 +
              \left\|\b x-\b v'\right\|_2^2}{2\sigma^2}: \sigma>0, \b v\in
          \X,\b v'\in\X\right\} \subseteq S=span\left(\b x \mapsto \left\|\b
          x\right\|_2^2, \{ \b x \mapsto x_i \}_{i=1}^d \hspace{0.15cm}, \b x
          \mapsto 1\right)$, $VC(\F_2)\le d+4$.
        \item $\F_3$: Since 
            \begin{align*}
             \F_3 = \left\{(\b x,\b y) \mapsto k(\b x,\b v)k(\b
                y,\b v') = e^{-\frac{\left\|\b x-\b v\right\|^2 + \left\|\b y -
                \b v'\right\|_2^2}{2\sigma^2}}=e^{-\frac{\left\|\left[\b x;\b
                y\right ]- \left[\b v;\b
                v'\right]\right\|_2^2}{2\sigma^2}}:\sigma >0, \b v\in \R^d, \b
            v'\in \R^d\right\},
            \end{align*}
             from the result on $\F_1$ we get that
            $VC(\F_3)\le 2d+4$. 
      \end{itemize}

  \textbf{Uniformly bounded, separable Carath{\'e}odory family: } \\
	The result follows from Lemma~\ref{lemma:F_i-properties} by noting that $\left|k(\b x,\b y)\right| \le 1=:B$, $(\b x,\b y)\mapsto e^{-\frac{\left\|\b x-\b y\right\|_2^2}{2\sigma^2}}$ is continuous ($\forall \sigma>0$), 
	  $\R^{+}$ is separable, and 
	  the $(\sigma,\b v) \mapsto e^{-\frac{\left\|\b x - \b v\right\|_2^2}{2\sigma^2}}$, $(\sigma,\b v, \b v') \mapsto e^{-\frac{\left\|\b x - \b v\right\|_2^2}{2\sigma^2}}e^{-\frac{\left\|\b x - \b v'\right\|_2^2}{2\sigma^2}}$, $(\sigma, \b v, \b v')\mapsto e^{-\frac{\left\|\b x - \b v\right\|_2^2}{2\sigma^2}} e^{-\frac{\left\|\b y - \b v'\right\|_2^2}{2\sigma^2}}$ mappings are continuous ($\forall \b x, \b y \in \X$).
\end{proof}

\begin{lem}[$\F_i$-s are VC-subgraph and uniformly bounded separable Carath{\'e}odory families for full Gaussian kernel]\label{lemma:full-Gaussian-properties}

   Let $\K=\{k_{\b A}: (\b x,\b y)\in \X \times \X \subseteq \R^d \times \R^d
   \mapsto e^{-(\b x-\b y)\T\b A (\b x -\b y)}: \b A\succeq \b 0\}$. Then the
   $\F_1$, $\F_2$, $\F_3$ classes 
      [see Eqs.~\eqref{eq:F_1-2}-\eqref{eq:F_3}] associated to $\K$ are
    \begin{itemize}
    \item VC-subgraphs with indices $VC(\F_1)\le \frac{d(d+1)}{2}+d+2$,
        $VC(\F_2)\le \frac{d(d+1)+2}{2} +d + 2$, $VC(\F_3)\le d(d+1)+2d +3$,
    \item uniformly bounded separable Carath{\'e}odory families, with $\left\|f\right\|_{L^\infty(\M)}\le 1$  for all $f\in \{\F_1, \F_2, \F_3 \}$.\textsuperscript{\ref{footnote:M-def}} 
    \end{itemize} 
\end{lem}

\begin{proof}
    We prove the VC index values; the rest is essentially identical to the proof of Lemma~\ref{lemma:isotropic-Gaussian-properties}.
  \begin{itemize}
      \item $\F_1$: Using that  $\G=\left\{\b x \mapsto (\b x-\b v)\T\b{A}(\b x
      -\b v): \b A \succeq \b 0, \b v\in \X\right\} \subseteq S:=span\left(
      \{ \b x \mapsto x_i x_j \}_{1\le i\le j\le d },  \{ \b x \mapsto x_i \}_{1\le
      i\le d }, \b x \mapsto 1 \right) $, we have $VC(\F_1)\le
      VC(\G)\le dim(S)+2 \le \frac{d(d+1)}{2} + d + 3$.
      \item $\F_2$: Since $\F_2 = \left\{\b x \mapsto k(\b x, \b v)k (\b x, \b v') = e^{-\left[(\b x-\b v)\T \b A(\b x-\b v) + \left(\b x-\b v'\right)\T \b A\left(\b x-\b v'\right)\right]}: \b A\succeq \b 0, \b v \in \X, \b v' \in \X\right\}$, and 
          \begin{align*}
            & \left\{(\b x, \b y) \mapsto (\b x-\b v)\T \b A(\b x-\b v) +
            \left(\b x-\b v'\right)\T \b A\left(\b x-\b v'\right)\right\}
            \subseteq S\\
            & := span\left(
      \{ \b x
      \mapsto x_i x_j \}_{1\le i\le j\le d },  \{ \b x \mapsto x_i \}_{1\le
      i\le d }, \b x \mapsto 1 \right),  
          \end{align*}
we have $VC(\F_2) \le VC(S) = dim(S)+2 \le \frac{d(d+1)}{2} + d + 3$.
      \item $\F_3$: Exploiting that
          \begin{equation*}
        \F_3 = \left\{(\b x,\b y) \mapsto k(\b x,\b v)k(\b y,\b v') =
          e^{-\left[(\b x-\b v)\T \b A(\b x-\b v)+\left(\b y-\b v'\right)\T \b
          B \left(\b y-\b v'\right)\right]}:\b A \succeq \b 0,\b B \succeq \b
      0, \b v\in \X, \b v'\in \X\right\},
          \end{equation*}
      and $\left\{(\b x, \b y) \mapsto (\b
      x-\b v)\T \b A(\b x-\b v)+(\b y-\b v')\T \b B (\b y-\b
  v')\right\}\subseteq S:=span\left(
      \{ (\b x, \b y) \mapsto x_i x_j \}_{1\le i\le j\le d },  \{  (\b x, \b y)
      \mapsto x_i \}_{1\le i\le d },  (\b x, \b y) \mapsto 1, \{  (\b x, \b y)
          \mapsto y_i y_j \}_{1\le i\le j\le d },  \{  (\b x, \b y)  \mapsto y_i \}_{1\le
          i\le d } \right) $,
    we have 
	$VC(\F_3) \le VC(S) = dim(S) + 2 \le  d(d+1) + 2d + 3$. 
  \end{itemize}
\end{proof}

\section{Proof of proposition\,\ref{prop:lb_me_power}  \label{sec:proof_lb_me_power}}

Recall Proposition\,\ref{prop:lb_me_power}: \lbmepower*

\subsection{Proof }

By (\ref{eq:lamb_minus_lamb_h2}), we have
\begin{align}
|\hat{\lambda}_{n}-\lambda_{n}| & \le\frac{\overline{c}_{1}n}{\gamma_{n}}\|\boldsymbol{\Sigma}-\mathbf{S}_{n}\|_{F}+\overline{c}_{2}n\|\overline{\mathbf{z}}_{n}-\boldsymbol{\mu}\|_{2}+\overline{c}_{3}n\gamma_{n}.\label{eq:lambh_minus_lamb}
\end{align}
We will bound each of the three terms in (\ref{eq:lambh_minus_lamb}).

\subsection*{Bounding $\|\overline{\mathbf{z}}_{n}-\boldsymbol{\mu}\|_{2}$ (second
term in (\ref{eq:lambh_minus_lamb}))}

Let $g(\mathbf{x},\mathbf{y},\mathbf{v}):=k(\mathbf{x},\mathbf{v})-k(\mathbf{y},\mathbf{v})$.
Define $\mathbf{v}^{*}:=\arg\max_{\mathbf{v}\in\{\mathbf{v}_{1},\ldots,\mathbf{v}_{J}\}}\left|\frac{1}{n}\sum_{i=1}^{n}g(\mathbf{x}_{i},\mathbf{y}_{i},\mathbf{v})-\mathbb{E}_{\mathbf{xy}}\left[g(\mathbf{x},\mathbf{y},\mathbf{v})\right]\right|$.
Define $G_{i}:=g(\mathbf{x}_{i},\mathbf{y}_{i},\mathbf{v}^{*})$.

\begin{align*}
\|\overline{\mathbf{z}}_{n}-\boldsymbol{\mu}\|_{2} & =\sup_{\mathbf{b}\in B(1,\mathbf{0})}\left\langle \mathbf{b},\overline{\mathbf{z}}_{n}-\boldsymbol{\mu}\right\rangle _{2}\\
 & \le\sup_{\mathbf{b}\in B(1,\mathbf{0})}\sum_{j=1}^{J}|b_{j}|\left|\frac{1}{n}\sum_{i=1}^{n}\left[k(\mathbf{x}_{i},\mathbf{v}_{j})-k(\mathbf{y}_{i},\mathbf{v}_{j})\right]-\mathbb{E}_{\mathbf{xy}}\left[k(\mathbf{x},\mathbf{v}_{j})-k(\mathbf{y},\mathbf{v}_{j})\right]\right|\\
 & =\sup_{\mathbf{b}\in B(1,\mathbf{0})}\sum_{j=1}^{J}|b_{j}|\left|\frac{1}{n}\sum_{i=1}^{n}g(\mathbf{x}_{i},\mathbf{y}_{i},\mathbf{v}_{j})-\mathbb{E}_{\mathbf{xy}}\left[g(\mathbf{x},\mathbf{y},\mathbf{v}_{j})\right]\right|\\
 & \le\left|\frac{1}{n}\sum_{i=1}^{n}G_{i}-\mathbb{E}_{\mathbf{xy}}\left[G_{1}\right]\right|\sup_{\mathbf{b}\in B(1,\mathbf{0})}\sum_{j=1}^{J}|b_{j}|\\
 & \le\sqrt{J}\left|\frac{1}{n}\sum_{i=1}^{n}G_{i}-\mathbb{E}_{\mathbf{xy}}\left[G_{1}\right]\right|\sup_{\mathbf{b}\in B(1,\mathbf{0})}\|\mathbf{b}\|_{2}\\
 & =\sqrt{J}\left|\frac{1}{n}\sum_{i=1}^{n}G_{i}-\mathbb{E}_{\mathbf{xy}}\left[G_{1}\right]\right|,
\end{align*}
where we used the fact that $\|\mathbf{b}\|_{1}\le\sqrt{J}\|\mathbf{b}\|_{2}$.
It can be seen that $-2B\le G_{i}\le2B$ because 
\[
G_{i}=k(\mathbf{x}_{i},\mathbf{v}^{*})-k(\mathbf{y}_{i},\mathbf{v}^{*})\leq|k(\mathbf{x}_{i},\mathbf{v}^{*})|+|k(\mathbf{y}_{i},\mathbf{v}^{*})|\le2B.
\]
Using Hoeffding's inequality (Lemma\,\ref{lem:hoeffding}) to bound
$\left|\frac{1}{n}\sum_{i=1}^{n}G_{i}-\mathbb{E}_{\mathbf{xy}}[G_{1}]\right|$,
we have 
\begin{align}
\mathbb{P}\left(n\overline{c}_{2}\|\overline{\mathbf{z}}_{n}-\boldsymbol{\mu}\|_{2}\le\alpha\right) & \ge1-2\exp\left(-\frac{\alpha^{2}}{8B^{2}\overline{c}_{2}^{2}Jn}\right).\label{eq:lbpow_mean_bound}
\end{align}

\subsection*{Bounding first ($\|\boldsymbol{\Sigma}-\mathbf{S}_{n}\|_{F}$) and
third terms in (\ref{eq:lambh_minus_lamb})}

Let $\eta(\mathbf{v}_{i},\mathbf{v}_{j}):=\left|\frac{1}{n}\sum_{a=1}^{n}g(\mathbf{x}_{a},\mathbf{y}_{a},\mathbf{v}_{i})g(\mathbf{x}_{a},\mathbf{y}_{a},\mathbf{v}_{j})-\mathbb{E}_{\mathbf{xy}}\left[g(\mathbf{x},\mathbf{y},\mathbf{v}_{i})g(\mathbf{x},\mathbf{y},\mathbf{v}_{j})\right]\right|$.
Define $(\mathbf{v}_{1}^{*},\mathbf{v}_{2}^{*})=\arg\max_{(\mathbf{v}^{(1)},\mathbf{v}^{(2)})\in\{(\mathbf{v}_{i},\mathbf{v}_{j})\}_{i,j=1}^{J}}\eta(\mathbf{v}^{(1)},\mathbf{v}^{(2)})$.
Define $H_{i}:=g(\mathbf{x}_{i},\mathbf{y}_{i},\mathbf{v}_{1}^{*})g(\mathbf{x}_{i},\mathbf{y}_{i},\mathbf{v}_{2}^{*})$.
By \eqref{eq:cov-decomp}, we have 
\begin{align*}
\|\mathbf{S}_{n}-\boldsymbol{\Sigma}\|_{F} & \le(*_{1})+(*_{2}),\\
(*_{1}) & =\bigg\|\frac{1}{n}\sum_{a=1}^{n}\mathbf{z}_{a}\mathbf{z}_{a}^{\top}-\mathbb{E}_{\mathbf{x}\mathbf{y}}[\mathbf{z}_{1}\mathbf{z}_{1}^{\top}]\bigg\|_{F},\\
(*_{2}) & =\frac{8B^{2}J}{n-1}+2B_{k}\sqrt{J}\frac{2n-1}{n-1}\|\overline{\mathbf{z}}_{n}-\boldsymbol{\mu}\|_{2}.
\end{align*}
We can upper bound $(*_{2})$ by applying Hoeffding's inequality to
bound $\|\overline{\mathbf{z}}_{n}-\boldsymbol{\mu}\|_{2}$ giving
\begin{align}
\mathbb{P}\left(\frac{\overline{c}_{1}n}{\gamma_{n}}(*_{2})\le\alpha\right) & \ge1-2\exp\left(-\frac{(\alpha\gamma_{n}-\alpha\gamma_{n}n+8B^{2}\overline{c}_{1}Jn)^{2}}{32B^{4}\overline{c}_{1}^{2}J^{2}n(2n-1)^{2}}\right).\label{eq:lbpow_cov2_bound}
\end{align}
We can upper bound $(*_{1})$ with 
\begin{align*}
(*_{1}) & =\sup_{\mathbf{B}\in B(1,\mathbf{0})}\left\langle \mathbf{B},\frac{1}{n}\sum_{a=1}^{n}\mathbf{z}_{a}\mathbf{z}_{a}^{\top}-\mathbb{E}_{\mathbf{x}\mathbf{y}}[\mathbf{z}_{1}\mathbf{z}_{1}^{\top}]\right\rangle _{F}\\
 & \le\sup_{\mathbf{B}\in B(1,\mathbf{0})}\sum_{i=1}^{J}\sum_{j=1}^{J}|B_{ij}|\left|\frac{1}{n}\sum_{a=1}^{n}g(\mathbf{x}_{a},\mathbf{y}_{a},\mathbf{v}_{i})g(\mathbf{x}_{a},\mathbf{y}_{a},\mathbf{v}_{j})-\mathbb{E}_{\mathbf{xy}}\left[g(\mathbf{x},\mathbf{y},\mathbf{v}_{i})g(\mathbf{x},\mathbf{y},\mathbf{v}_{j})\right]\right|\\
 & \le\left|\frac{1}{n}\sum_{a=1}^{n}H_{a}-\mathbb{E}_{\mathbf{xy}}\left[H_{1}\right]\right|\sup_{\mathbf{B}\in B(1,\mathbf{0})}\sum_{i=1}^{J}\sum_{j=1}^{J}|B_{ij}|\\
 & \le J\left|\frac{1}{n}\sum_{a=1}^{n}H_{a}-\mathbb{E}_{\mathbf{xy}}\left[H_{1}\right]\right|\sup_{\mathbf{B}\in B(1,\mathbf{0})}\|\mathbf{B}\|_{F}=J\left|\frac{1}{n}\sum_{a=1}^{n}H_{a}-\mathbb{E}_{\mathbf{xy}}\left[H_{1}\right]\right|,
\end{align*}
where we used the fact that $\sum_{i=1}^{J}\sum_{j=1}^{J}|B_{ij}|\le J\|\mathbf{B}\|_{F}$.
It can be seen that $-4B^{2}\le H_{a}\le4B^{2}$. Using Hoeffding's
inequality (Lemma\,\ref{lem:hoeffding}) to bound $\left|\frac{1}{n}\sum_{a=1}^{n}H_{a}-\mathbb{E}_{\mathbf{xy}}\left[H_{1}\right]\right|$,
we have
\begin{align}
\mathbb{P}\left(\frac{\overline{c}_{1}n}{\gamma_{n}}(*_{1})\le\alpha\right) & \ge1-2\exp\left(-\frac{\alpha^{2}\gamma_{n}^{2}}{32B^{4}J^{2}\overline{c}_{1}^{2}n}\right),\label{eq:lbpow_cov1_bound}
\end{align}
implying that 
\begin{equation}
\mathbb{P}\left(\frac{\overline{c}_{1}n}{\gamma_{n}}(*_{1})+\overline{c}_{3}n\gamma_{n}\le\alpha\right)\ge1-2\exp\left(-\frac{\left(\alpha-\overline{c}_{3}n\gamma_{n}\right)^{2}\gamma_{n}^{2}}{32B^{4}J^{2}\overline{c}_{1}^{2}n}\right).\label{eq:lbpow_cov1_gam}
\end{equation}
Applying a union bound on (\ref{eq:lbpow_mean_bound}), (\ref{eq:lbpow_cov2_bound}),
and (\ref{eq:lbpow_cov1_gam}) with $t=\alpha/3$, we can conclude
that 
\begin{align*}
 & \mathbb{P}\left(\left|\hat{\lambda}_{n}-\lambda_{n}\right|\le t\right)\ge\mathbb{P}\left(\frac{\overline{c}_{1}n}{\gamma_{n}}\|\boldsymbol{\Sigma}-\mathbf{S}_{n}\|_{F}+\overline{c}_{2}n\|\overline{\mathbf{z}}_{n}-\boldsymbol{\mu}\|_{2}+\overline{c}_{3}n\gamma_{n}\le t\right)\\
 & \ge1-2\exp\left(-\frac{t^{2}}{3^{2}\cdot8B^{2}\overline{c}_{2}^{2}Jn}\right)-2\exp\left(-\frac{(t\gamma_{n}n-t\gamma_{n}-24B^{2}\overline{c}_{1}Jn)^{2}}{3^{2}\cdot32B^{4}\overline{c}_{1}^{2}J^{2}n(2n-1)^{2}}\right)-2\exp\left(-\frac{\left(t/3-\overline{c}_{3}n\gamma_{n}\right)^{2}\gamma_{n}^{2}}{32B^{4}J^{2}\overline{c}_{1}^{2}n}\right).
\end{align*}
A rearrangement yields{\small 
\begin{align*}
 & \mathbb{P}\left(\hat{\lambda}_{n}\ge T_{\alpha}\right)\\
 & \ge1-2\exp\left(-\frac{(\lambda_{n}-T_{\alpha})^{2}}{3^{2}\cdot8B^{2}\overline{c}_{2}^{2}Jn}\right)-2\exp\left(-\frac{(\gamma_{n}(\lambda_{n}-T_{\alpha})(n-1)-24B^{2}\overline{c}_{1}Jn)^{2}}{3^{2}\cdot32B^{4}\overline{c}_{1}^{2}J^{2}n(2n-1)^{2}}\right)-2\exp\left(-\frac{\left((\lambda_{n}-T_{\alpha})/3-\overline{c}_{3}n\gamma_{n}\right)^{2}\gamma_{n}^{2}}{32B^{4}J^{2}\overline{c}_{1}^{2}n}\right).
\end{align*}
}Define $\xi_{1}:=\frac{1}{3^{2}\cdot8B^{2}\overline{c}_{2}^{2}J},\xi_{2}:=24B^{2}\overline{c}_{1}J,\xi_{3}:=3^{2}\cdot32B^{4}\overline{c}_{1}^{2}J^{2},\xi_{4}:=32B^{4}J^{2}\overline{c}_{1}^{2}$.
We have {\small{}
\begin{align*}
 & \mathbb{P}\left(\hat{\lambda}_{n}\ge T_{\alpha}\right)\\
 & \ge1-2\exp\left(-\frac{\xi_{1}(\lambda_{n}-T_{\alpha})^{2}}{n}\right)-2\exp\left(-\frac{(\gamma_{n}(\lambda_{n}-T_{\alpha})(n-1)-\xi_{2}n)^{2}}{\xi_{3}n(2n-1)^{2}}\right)-2\exp\left(-\frac{\left((\lambda_{n}-T_{\alpha})/3-\overline{c}_{3}n\gamma_{n}\right)^{2}\gamma_{n}^{2}}{\xi_{4}}\right).
\end{align*}
}{\small \par}

\QEDB

\section{External lemmas}

In this section we detail some external lemmas used in our proof.
\begin{lem}[properties of VC classes, see page~141, 146-147 in \cite{Vaart2000} and page~160-161 in \cite{Kosorok2008}]~\label{lemma:VC-properties}
    \begin{itemize}
	  \item[(i)] Monotonicity: $\G \subseteq \tilde{\G}\subseteq L^0(\M) \Rightarrow VC(\G)\le VC(\tilde{\G})$.
	  \item[(ii)] Finite-dimensional vector space: if $\G$ is a finite-dimensional vector space of measurable functions, then $VC(\G)\le dim(\G)+2$.
	  \item[(iii)] Composition with monotone function: If $\G$ is VC and $\phi:\R \rightarrow \R$ is monotone, then for $\phi \circ \G := \{\phi \circ g: g\in \G\}$, $VC(\phi\circ \G)\le VC(\G)$.
	  \item[(iv)] The $r$-covering number of a VC class grows only polynomially in $\frac{1}{r}$: Let $\F$ be VC on the domain $\M$ with measurable envelope $F$ ($|f(m)|\le F(m)$, $\forall m\in \M, f\in \F$). Then for any $q\ge 1$ and $\Me$ probability measure for which 
	  $\left\|F\right\|_{L^q(\M,\Me)}>0$
	    \begin{align}
		N\left(r \left\|F\right\|_{L^q(\M,\Me)},\F,L^q(\M,\Me)\right) &\le C \times VC(\F) (16e)^{VC(\F)} \left(\frac{1}{r}\right)^{q[VC(\F)-1]}
	    \end{align}
	    for any $r\in (0,1)$ with a universal constant $C$.
    \end{itemize}
\end{lem}

\begin{lem}[McDiarmid's inequality]\label{lemma:McDiarmid}
Let $X_1, \ldots, X_n\in \M$ be independent random variables and let $g: \M^n\rightarrow \R$ be a function such that the 
\begin{align}
  \sup_{\b x_1,\ldots,\b x_n, \b x_j' \in \M}\left|g(\b x_1,\ldots,\b x_n) - g\left(\b x_1,\ldots,\b x_j,\b x_j',\b x_{j+1},\ldots,\b x_n\right)\right| \le b\label{eq:bounded-diff-property}
 \end{align}
bounded difference property holds. Then for arbitrary $\delta\in (0,1)$
\begin{align*}
  \P\left(g(X_1,\ldots, X_n) \le  \E [g(X_1,\ldots, X_n)] + b\sqrt{\frac{\log\left(\frac{1}{\delta}\right) n}{2}}\right) \ge 1- \delta.
\end{align*}
\end{lem}

\begin{lem}[Hoeffding's inequality]
\label{lem:hoeffding} Let $X_{1},\ldots,X_{n}$ be i.i.d. random
variables with $\mathbb{P}(a\le X_{i}\le b)=1$. Let $\overline{X}:=\frac{1}{n}\sum_{i=1}^{n}X_{i}$.
Then,
\[
\mathbb{P}\left(\left|\overline{X}-\mathbb{E}[\overline{X}]\right|\le t\right)\ge1-2\exp\left(-\frac{2nt^{2}}{(b-a)^{2}}\right).
\]
Equivalently, for any $\delta\in(0,1)$, with probability at least
$1-\delta$, it holds that
\[
\left|\overline{X}-\mathbb{E}[\overline{X}]\right|\le\frac{b-a}{\sqrt{2n}}\sqrt{\log(2/\delta)}.
\]
\end{lem}

\end{document}